\documentclass{llncs}

\usepackage{times}

\usepackage{mathtools} 
\usepackage{scalerel}[2016/12/29] 
\usepackage[margin=1.7in]{geometry}   
\usepackage{changepage}
\usepackage{dingbat} 
\usepackage{enumitem} 
\usepackage{xcolor} 
\usepackage{amssymb}              
\usepackage{caption}  
\usepackage{tikz}
\usetikzlibrary{arrows.meta}  
\usepackage{tcolorbox} 
\tcbuselibrary{raster,skins}
\usepackage{rotating}    
\newcommand{\MMA}{\textsf{MMA}}
\newcommand{\update}{{\small \textsf{updt}}}
\newcommand{\detect}{{\small \textsf{dtct}}}   
\newcommand{\sub}{{\small \textsf{rstrct}}} 

\newcommand{\adjust}{{\small \textsf{adjst}}}
 
\newcommand{\announce}{{\small \textsf{annce}}}
\newcommand{\Apub}{\ensuremath{A_{\textsf{pub}}}} 
\newcommand{\Rpub}{\ensuremath{R_{\textsf{pub}}}} 
\newcommand{\om}{{\small \textsf{om}}}

\newcommand{\sem}{{\small \textsf{Sem}}} 
\newcommand{\semsmall}{{\small \textsf{sem}}}

\newcommand{\getArg}{{\small \textsf{getArg}}}
\newcommand{\getR}{{\small \textsf{getR}}}

\newcommand{\revV}{{\small \textsf{revV}}}

\newcommand{\Fg}{\ensuremath{F^{\Dung}}_{\textsf{G}}}
\newcommand{\Fpub}{\ensuremath{F^{\Dung}}_{\textsf{pub}}}  
\newcommand{\pub}{\textsf{pub}}

\newcommand{\fe}{\ensuremath{h_{\textbf{E}}}}  
\newcommand{\ve}{\ensuremath{v_{\textbf{E}}}}
\newcommand{\fleq}{\ensuremath{f_{\leq_{\textbf{p}}}}} 
\newcommand{\fleqinter}{\ensuremath{r_{\leq_{\textbf{p}}}^{\textsf{inter}}}} 
\newcommand{\fleqintra}{\ensuremath{\fleq^{\textsf{intra}}}}
\newcommand{\fa}{\ensuremath{f_{\textbf{A}}}} 
\newcommand{\fsem}{\ensuremath{f_{\semsmall}}}
\newcommand{\gsem}{\ensuremath{g_{\semsmall}}}

\newcommand{\co}{\textsf{co}}
\newcommand{\pr}{\textsf{pr}}
\newcommand{\gr}{\textsf{gr}}

\newcommand{\Dung}{\textsf{D}}

\newcommand{\Sem}{\textsf{Sem}}

\newcommand{\hide}[1]{} 
  
\newcommand{\orC}{\textsf{or}}
\newcommand{\notC}{\textsf{not}}
\newcommand{\andC}{\textsf{and}}

\definecolor{darkGray}{RGB}{64,64,64}

\begin{document}
     \title{Formulating  
     Manipulable Argumentation\\
     with Intra-/Inter-Agent Preferences} 

\author{\institute{}}
\author{Ryuta Arisaka \and Makoto Hagiwara \and 
    Takayuki Ito
    \institute{Nagoya Institute of Technology, Nagoya, Japan\\
    email: ryutaarisaka@gmail.com, lastname.firstname@nitech.ac.jp}} 
    \maketitle 
\vspace{-0.6cm} 
\begin{abstract}   
  From marketing to politics, 
  exploitation of incomplete information 
  through 
  selective communication of arguments  
  is ubiquitous. 
  In this work, we focus on development of 
  an argumentation-theoretic model 
  for manipulable multi-agent argumentation, 
  where each agent may transmit deceptive 
  information to others for tactical motives. 
  In particular, we study characterisation of 
  epistemic states, and their roles in deception/honesty detection 
  and (mis)trust-building.   
  To this end, we propose the use of intra-agent preferences to 
  handle deception/honesty
  detection and inter-agent preferences to 
  determine which agent(s) to believe in more. 
  We show how deception/honesty in an argumentation of an 
  agent, if detected, 
  would alter the agent's perceived trustworthiness, 
  and how that may affect their judgement as to which arguments 
  should be acceptable. 
\end{abstract} 
\vspace{-0.6cm} 
\section{Introduction}\label{section_introduction} 
 To adequately characterise multi-agent argumentation, 
 it is important to 
 model what an agent 
 sees of other agents' 
 argumentations ({\it Epistemic 
 Aspect}). 
 It is also important
 to model 
 how agents interact with 
  others 
 ({\it Agent-to-Agent
  Interaction}). 
  These two factors determine dynamics of 
  multi-agent argumentation, and are thus 
   central to: argumentation-based negotiations 
 (Cf. two surveys 
 \cite{Rahwan03,Dimopoulos14});  
 strategic dialogue games that 
 may involve 
 opponent modelling, 
e.g. \cite{Kakas05,Sakama12,Grossi13,Rahwan09,Parsons05,Riveret08,Hadjinikolis13,Hadoux15,Hadoux17} 
(Cf. also \cite{Thimm14} for a survey till 2013-2014); 
 defence outsourcing \cite{ArisakaBistarelli18}; and 
 synchronised decision-adjustment by agents 
  \cite{Rienstra11,AST17,Thimm14,Rienstra13}. 
  In this work, we consolidate them 
  for manipulable multi-agent argumentation, 
  where an agent may 
 announce to other agent(s) any argumentations he/she sees 
 fit for his/her tactical motives - including 
 potentially false ones. Despite 
 their importance in real-life 
 argumentation, 
 very few attempts 
 at
 modelling manipulable or deceptive argumentation 
 \cite{Takahashi16,Sakama12,Kuipers10,Kontarinis15}
 currently exist 
 in the literature of formal argumentation. 
 Agents' epistemic states, their
 roles in deception detection, and  their impacts on agent-to-agent
 interactions are still to be 
 studied further, especially in non-two-party multi-agent argumentation, 
  e.g. \cite{Grossi13,ArisakaBistarelli18,ArisakaSatoh18}, 
  which, too, is scarcely covered.  
\vspace{-0.4cm} 
\subsubsection{Exploitation 
is ubiquitous under incomplete information.} \label{subsection_information_manipulation}
 While honesty may be a moral virtue 
 from ancient times, 
  greater strategic advantages can be obtained in real life 
 by withholding disadvantageous information \cite{Sakama12,Hadjinikolis13,Rahwan08}, 
 half-truths \cite{Chomsky10}, and 
 through various other tactical ruses. In certain 
 circumstances, 
 even outright bluffing may work to 
 one's advantage if others do not detect 
 it, as seen in the game of Poker
 where a player is generally
 uncertain about the hands of the others, 
 or in the game of Mafia where, furthermore,  
 a player may not be
 certain of the role played 
 by another player - in particular, whether he/she is or 
 is not his/her opponent (see \url{https://en.wikipedia.org/wiki/Mafia\_(party\_game)} for detail). 
    
 For illustration of our formalism, 
 in this paper we draw examples from an end game of Mafia 
 involving  3 agents separated into two teams, Team Mafia and Team 
 Innocent, each trying to eliminate a player in the opponent 
 team for a win. 
 There is at least one player who does not know 
 which of the other players is in the same team, 
 so the other players need to convince the unsure 
  player through argumentations. Unlike the setting in \cite{Grossi13}, 
 however, the unsure player(s) are not bystanders; 
they are as much of a participant as the other(s) are.  
Moreover, 
 at least one of the two other players 
 is also uncertain about the role of the other players. 
 This game thus presents a great opportunity for 
 manipulable argumentation. Exact description is given in Section~\ref{section_motivation}.   
\vspace{-0.3cm} 
\subsubsection{Contributions}    
  are in the theory of formal   
  argumentation. Due to space limit, 
  experiments on strategies under protocol 
  are left to an extended work.  
  We present a 
  (generally non-two-party) manipulable multi-agent 
  argumentation with: agents' epistemic states; 
 {\it intra-agent preference relations}; and 
 {\it inter-agent preference relations}, for 
 characterising (1) detection of deception and honesty
 and (2) their impacts on perceived trustworthiness. 
 While deception detection in argumentation is discussed in 
 \cite{Sakama12,Takahashi16}, several assumptions are made 
 such as attack-omniscience (every agent
 knows every attack among the arguments it is aware of), 
 absence of ``recursive knowledge \cite{Oren09} (which 
 is standard in epistemic 
  logic \cite{Hendricks06})'' in general, and so on. 
  As we are to illustrate 
 in detail later in Section~\ref{section_motivation},  
 there are situations not handled by 
  the detection approach of 
  \cite{Sakama12}. With the two types of 
  preference relations, our formal model 
  {\it allows
   each agent to reason differently} when detecting 
  deception/honesty and when deciding which arguments 
  to publicly accept, which helps {\it refine deception detection}
  in the previously studies, and which, moreover, allows 
  the {\it impacts of deception/honesty on agents' perceived trustworthiness} 
  to be expressed within it. To the best of 
  our knowledge, these generalisations have not been
  undertaken. In view of research interests 
 in this game \cite{Hung10,Toriumi16}, our work should have a wider 
  implication outside formal argumentation.
   
\vspace{-0.2cm} 
\section{Technical Preliminaries} 　 
\vspace{-1.1cm}  
\subsubsection{Abstract argumentation}
\label{subsection_abstract_argumentation}  
considers 
an argumentation as a graph where a node 
is an argument and an edge is an attack 
of the source argument on the target argument.\cite{Dung95} 
Let $\mathcal{A}$ denote 
the class of abstract entities that we understand 
as arguments, then 
a (finite) Dung argumentation is a pair $(A, R)$ 
with $A \subseteq_{\text{fin}} \mathcal{A}$ 
and $R \subseteq A \times A$. 
We denote the class of all Dung argumentations 
by $\mathcal{F}^{\Dung}$. From here on, 
we denote: a member of $\mathcal{A}$ by $a$; 
a finite subset of $\mathcal{A}$ by $A$; 
and a member of $\mathcal{F}^{\Dung}$ 
by $F^{\Dung}$, all with or without a subscript. 
For any $(A, R) \equiv F^\Dung$, 
we denote by $2^{F^{\Dung}}$ 
the following set: $\{(A_1, R_1) \ | \ 
A_1 \subseteq A \ \andC \ R_1 \subseteq R \cap (A_1 \times A_1)\}$, 
i.e. all sub-Dung-argumentations of $F^{\Dung}$.\footnote{``$\andC$'' 
instead of ``and'' is used when the context in which 
the word appears strongly indicates classic-logic truth-value 
comparisons. Similarly for $\orC$ (disjunction) 
and $\notC$ (negation).}  

Assume that the following notations are for any chosen 
$(A, R) \in \mathcal{F}^{\Dung}$. $a_1 \in A$ 
is said to attack $a_2 \in A$ if and only if, or iff, $(a_1, a_2) \in R$. 
$A_1 \subseteq A$ is said to be conflict-free iff 
there is no $a_1, a_2 \in A$ such that $(a_1, a_2) \in R$. 
$A_1 \subseteq A$ is said to defend $a_x \in A$ 
iff every $a_y \in A$ attacking $a_x$ is attacked by 
at least one member of $A_1$. 
$A_1 \subseteq A$ is said to be: 
admissible iff $A_1$ is conflict-free and defends all its 
members; complete iff $A_1$ is admissible and 
includes every argument it defends; 
preferred iff $A_1$ is a maximally complete set; 
and grounded iff $A_1$ is the set intersection of 
all complete sets. 
Let $\Sem$ be $\{\co, \pr, \gr\}$, 
and let $\Dung: \Sem \times \mathcal{F}^{\Dung} \rightarrow 
2^{2^{\mathcal{A}}}$ be such that: 
$\Dung(\co, (A, R))$ is the set of 
all complete sets of $(A, R)$; 
$\Dung(\pr, (A, R))$ is the set of 
all preferred sets of $(A, R)$; 
and $\Dung(\gr, (A, R))$ is the set of 
all grounded sets of $(A, R)$. 
$\Dung(\co, (A, R))$, $\Dung(\pr, (A, R))$ and $\Dung(\gr, (A, R))$
are called the complete semantics, the preferred semantics and 
the grounded semantics of $(A, R)$. 
Clearly $|\Dung(\gr, (A, R))| = 1$. 
There are other semantics, and an interested 
reader is referred to  
 \cite{Baroni07} for an overview. 
For a chosen $\semsmall \in \Sem$, $a_1 \subseteq A$ 
is said to be: credulously acceptable iff  
there exists some $A_1 \in \Dung(s, (A, R))$ such that 
$a \in A_1$; and skeptically acceptable 
iff $a \in A_1$ for every $A_1 \in \Dung(s, (A, R))$. 
We may simply write $a (\in A)$ is acceptable in $\semsmall \in \Sem$ 
when $a$ is at least credulously acceptable 
in $\semsmall \in \Sem$. 
\subsubsection{Attack-reverse preference.}  \label{subsection_attack_reverse_preference}  
Suppose $(\{a_1, a_2\}, 
\{(a_1, a_2)\})$ with two arguments and 
an attack. 
For any member $\semsmall$ of $\Sem$,  
we obtain that $a_1$ but not $a_2$ is acceptable. 
Suppose, however,  that 
some agent observing 
this argumentation still prefers 
to accept $a_2$. 
The agent could conceive 
an extension of 
this argumentation, 
$(\{a_1, a_2, a_3\}, \{(a_1, a_2), 
(a_3, a_1)\})$ with 
some argument $a_3$: 
{\it I (= the agent) doubt it 
in the absence of any evidence.}, which attacks
$a_1$. For  $\semsmall \in \Sem$, $a_2$ (and $a_3$) but not $a_1$ then become 
acceptable. 

The same effect can be achieved 
without any auxiliary argument 
if we apply attack-reverse 
preference \cite{Amgoud14}.  
Assume a partial order $\leq_p$ over $A$ 
in some $(A, R)$,  
then 
$R' \subseteq A \times A$ is said to 
be $\leq_p$-adjusted $R$ 
iff it is the least set that 
satisfies the following 
conditions.  We assume that $a_1 <_p a_2$ iff 
$a_1 \leq_p a_2$ $\andC$ $\notC\ a_2 \leq_p a_1$.   
\vspace{-0.1cm} 
{\scriptsize  
\begin{itemize} 
   \item   $(a_1, a_2) \in R'$ 
   if 
   $(a_1, a_2) \in R$ 
    $\andC$ ($\notC$ 
        $a_1 <_p a_2$).  \quad\qquad 
    $-${\ }  $(a_2, a_1) \in R'$ 
    if $(a_1, a_2) \in R$ 
      $\andC$ $a_1 <_p a_2$. 
\end{itemize}  
}
\noindent By setting $\leq_p$ to be such that 
$a_1 <_p a_2$ in 
$(\{a_1, a_2\}, \{(a_1, a_2)\})$,
it is easy to see that 
$\leq_p$ expresses 
the agent's preference: 
under $\leq_p$-adjusted $\{(a_1, a_2)\}$, 
which is $\{(a_2, a_1)\}$, a semantics with 
some $\semsmall \in \Sem$ 
makes $a_2$ but not $a_1$ acceptable.  
$(A, R, \leq_{p})$ is said to be  
a Dung argumentation $(A, R)$ with 
a preference $\leq_p$.  
\vspace{-0.3cm} 
\subsubsection{Agent argumentation with 
epistemic functions.}\label{subsection_agent_argumentation}

Let $\mathcal{E}$ be a class of  
abstract entities that we understand as agents.  
Let $e$ refer to 
a member of  $\mathcal{E}$, and let $E$
refer to 
a finite subset of $\mathcal{E}$, each with or without 
a subscript.  
Meanwhile, 
let $\getArg: \mathcal{F}^{\Dung} \rightarrow 
      2^{\mathcal{A}}$ and 
    $\getR: \mathcal{F}^{\Dung} \rightarrow 
      2^{\mathcal{A} \times \mathcal{A}}$ be such that 
     $\getArg((A, R)) = A$, and that $\getR((A, R)) = R$ 
      for any $(A, R) \in \mathcal{F}^{\Dung}$. 
Then an 
agent argumentation with 
agents' local scopes and an epistemic 
function indicating their knowledge  
is expressed by 
$(F^{\Dung}, E, \fe, \fa, \fsem)$ where: 
$\fe: E \rightarrow (2^{F^{\Dung}} 
\backslash (\emptyset, \emptyset))$
is such that $\getR(\fe(e)) = \getR(F^{\Dung}) \cap (\getArg(\fe(e)) \times 
\getArg(\fe(e)))$, 
and that 
$\getArg(\fe(e_1)) \cap \getArg(\fe(e_2))  = \emptyset$ 
if $e_1 \not= e_2$; and where  
$\fa: E \rightarrow (2^{F^{\Dung}} \backslash (\emptyset, \emptyset))$ 
is 
such that $\fe(e) \in 2^{\fa(e)}$, and that 
$\getR(\fa(e)) \cap (\getArg(\fe(e)) \times \getArg(\fe(e))) 
= \getR(\fe(e))$ 
for $e \in E$. 
Meanwhile, $\fsem: E \rightarrow \sem$ indicates 
the type of semantics, $\fsem(e)$, which the agent $e$ adopts 
when computing acceptability semantics. 
The purpose of $\fe$ 
is to express agents' local scopes. 
$\fa(e)$ is the argumentation
$e$ is aware of, 
which naturally subsumes $\fe(e)$, 
and the attacks
in $\fa(e)$ match $\getR(\fe(e))$ 
exactly as far 
as $\getArg(\fe(e))$ are concerned. 
$\fa(e)$ for $e \in E$ is called 
local agent argumentation of $e$ in the global argumentation 
$F^{\Dung}$.
   \vspace{-0.2cm} 
    \section{Motivation for Epistemic 
    States and 
    Agent Preferences}\label{section_motivation}  
  We draw examples from an end game of Mafia.
  The setting is as follows with 3 agents left.\\

{\small 
 {\it Common knowledge among them.} 
 One agent is a killer, and 
 the other two agents are  
 civilians, of which 
 at most one can be a detective -   
 no player but detective itself, if in the game, knows for certain 
 that there 
 is a detective. 
 Team Mafia comprises just the killer, 
 and Team Innocent consists of the civilians.

 {\it Agents' knowledge.} 
 All three of them 
 know which role they have been assigned to. 
 Killer knows that the other two are 
 not a killer.  Detective, if in 
 Team Innocent, knows who the ordinary civilian 
 (to be simply described civilian hereafter) 
 and who the killer are by its ability. 
 It also knows the killer knows it is the killer, 
 and that the civilian knows it is a civilian. 
  However, 
 no civilians know the role of the other players.  
 
 {\it Argumentations.} Each agent may 
 entertain argumentations generally  consisting 
 of a set of arguments, e.g. ``Agent $e_1$ is Killer.'', 
 and attacks among them. 
 They may also announce argumentations publicly. 
 Arguments and attacks 
 in a public announcement may not be actual, e.g.  
 even if ``$e_2$ is Killer"
   is just a guess or known to be untrue to $e_3$, $e_3$ may 
  still put the argument forward, and similarly 
  for an attack. 
 Any argumentation announced publicly is known to 
 every agent. \\
 ${\ }\quad$ In the end, each agent chooses with its 
 own semantics  
 (that is, its own judgement criteria 
 to decide which arguments to accept)
 who the killer to be hanged is. 
 If there is an agent chosen 
 by the other two, then  
 the team the chosen agent belongs to loses, and 
 the team the chosen agent does not belong to wins. It is 
 everybody's interest to let its team win, to which end 
 they thus conduct argumentation. \\
}

\noindent Presuming this setting, we motivate  
epistemic states and intra-/inter-agent preferences, 
and how they are used for: deception/honesty detection; 
and updates on agent-to-agent trusts.  
\subsection{Epistemic states} \label{subsection_highlight1}     
\vspace{-0.1cm} 
 \begin{center}    
     \includegraphics[scale=0.11]{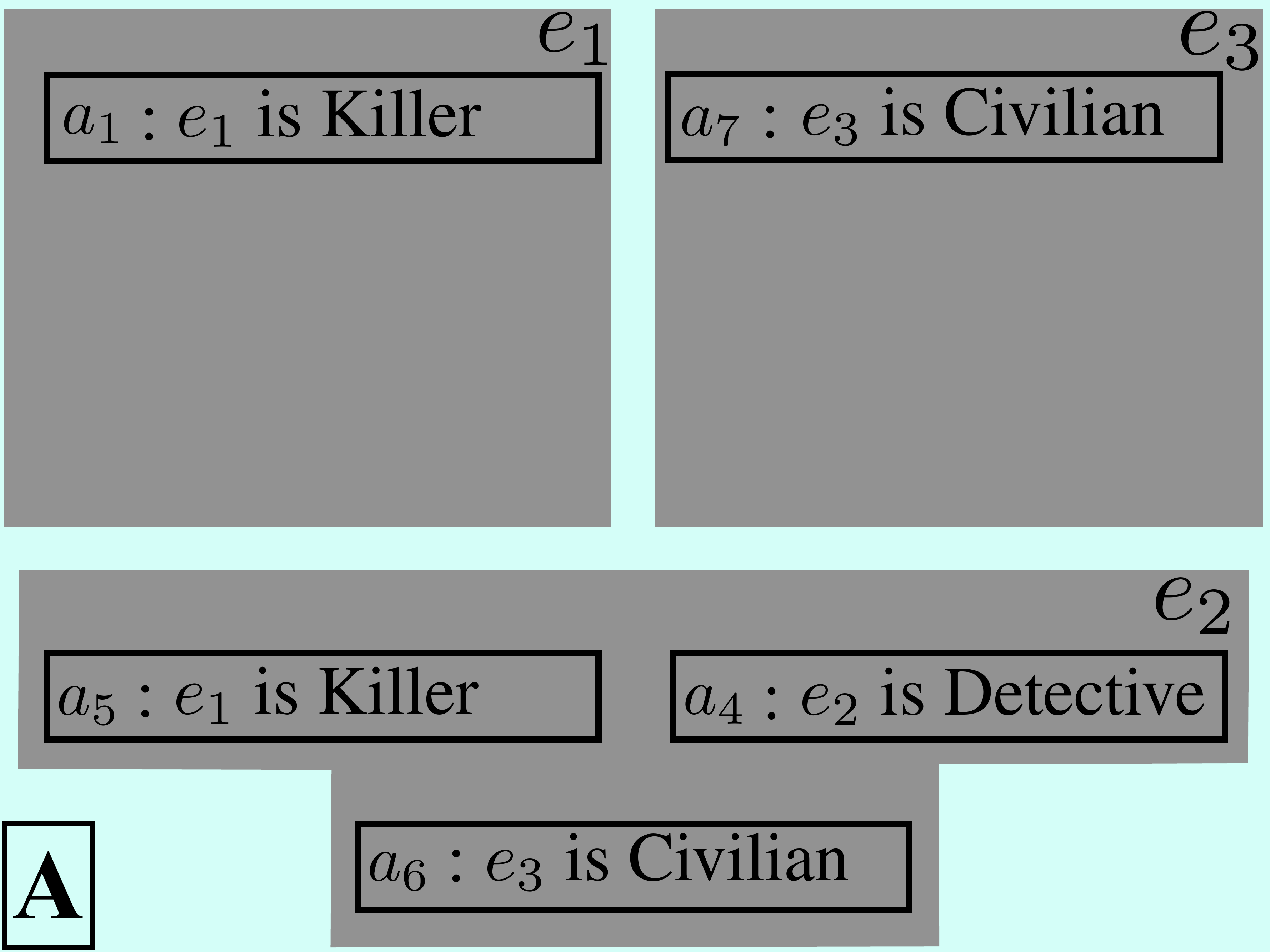}
     \includegraphics[scale=0.11]{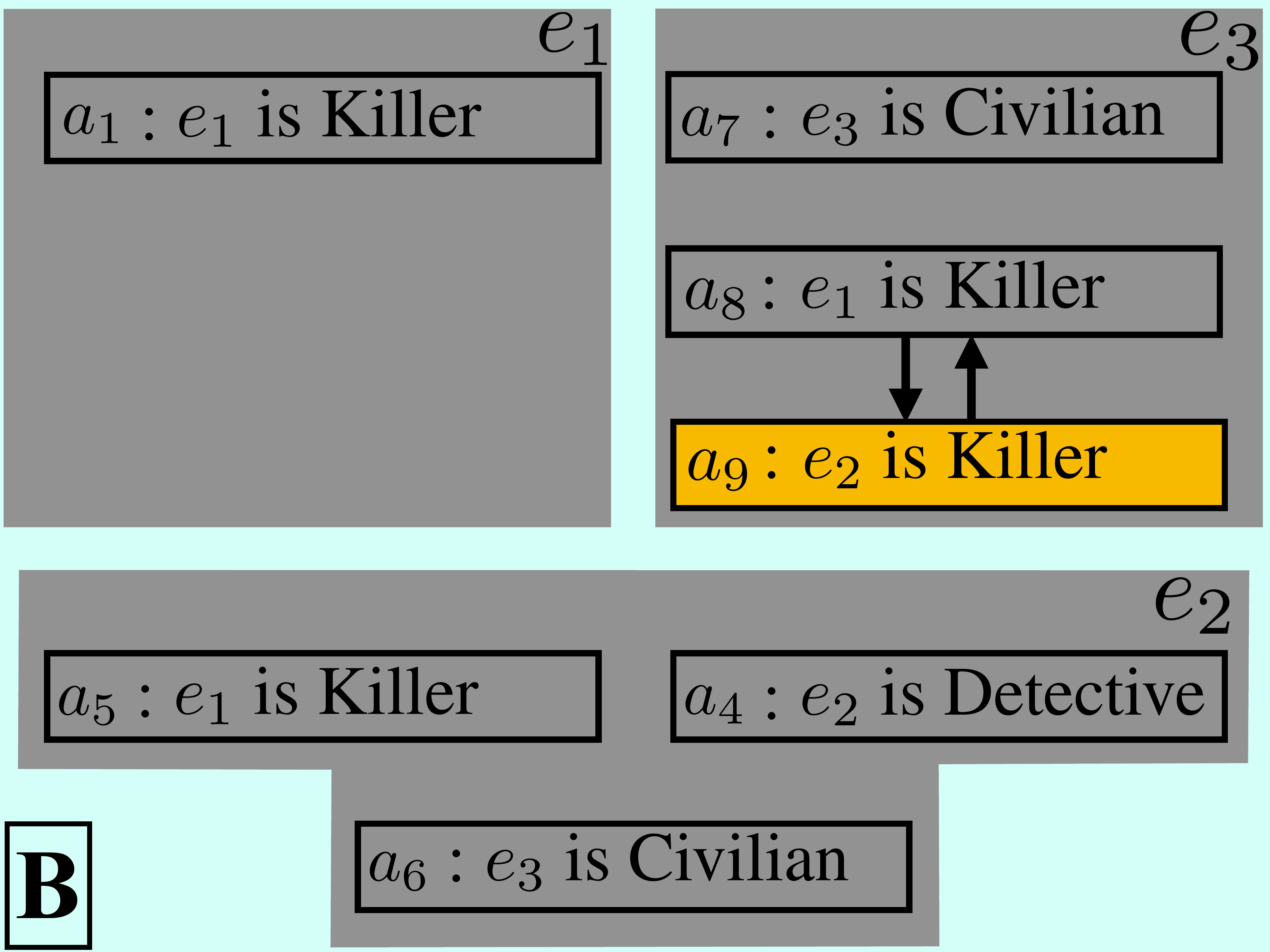}
     \includegraphics[scale=0.11]{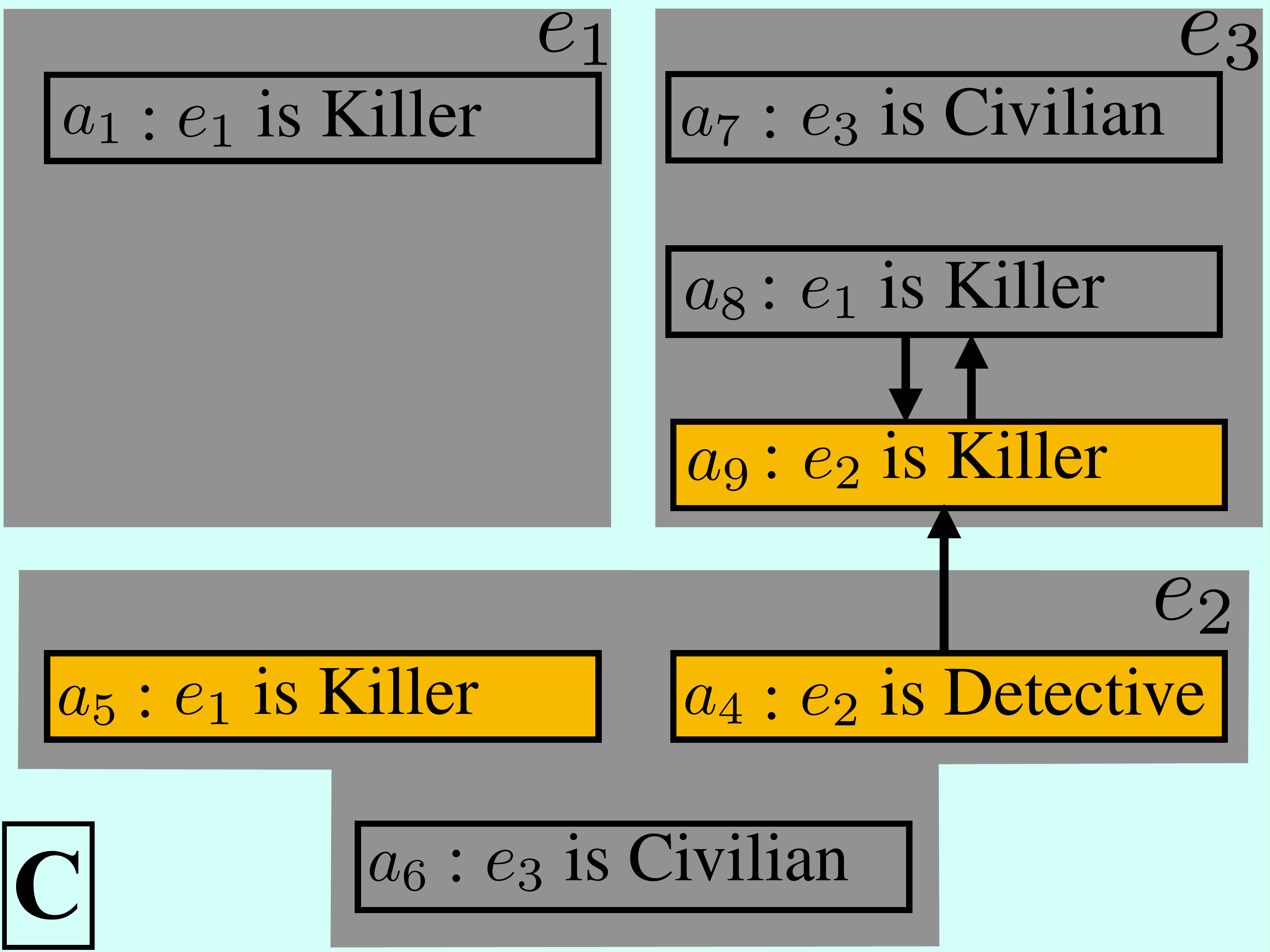}
     \includegraphics[scale=0.11]{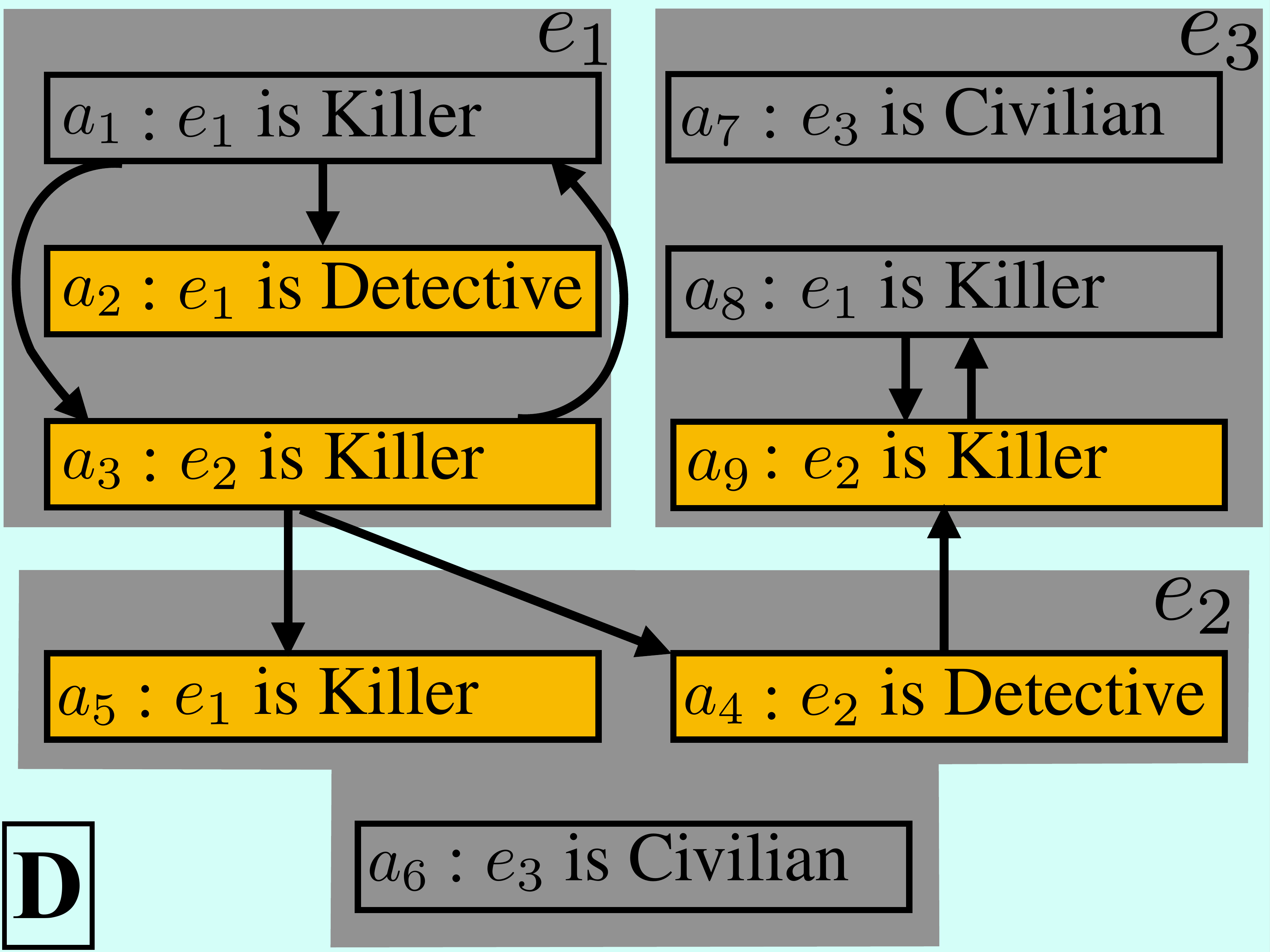}
     \includegraphics[scale=0.11]{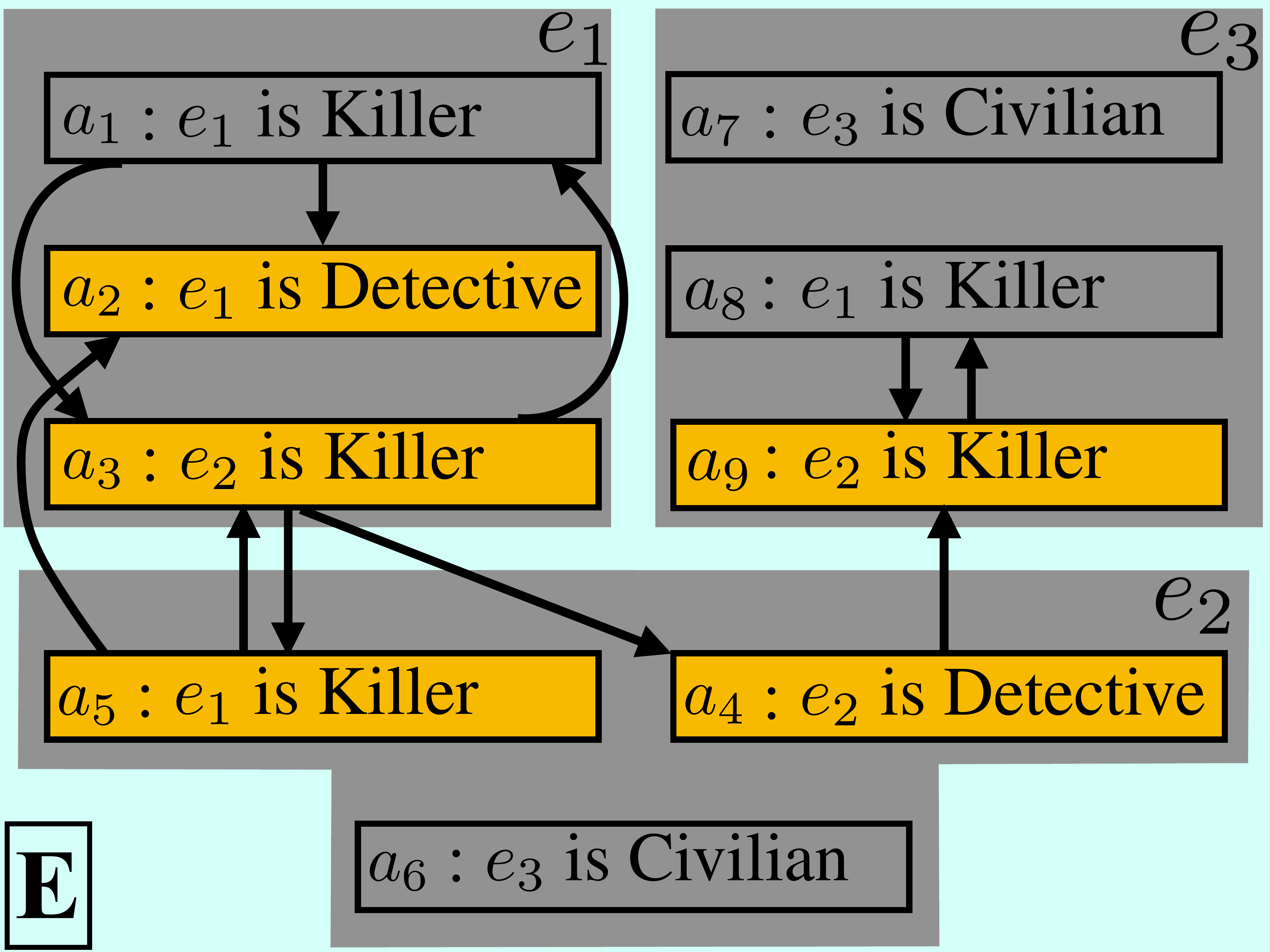}
     \includegraphics[scale=0.11]{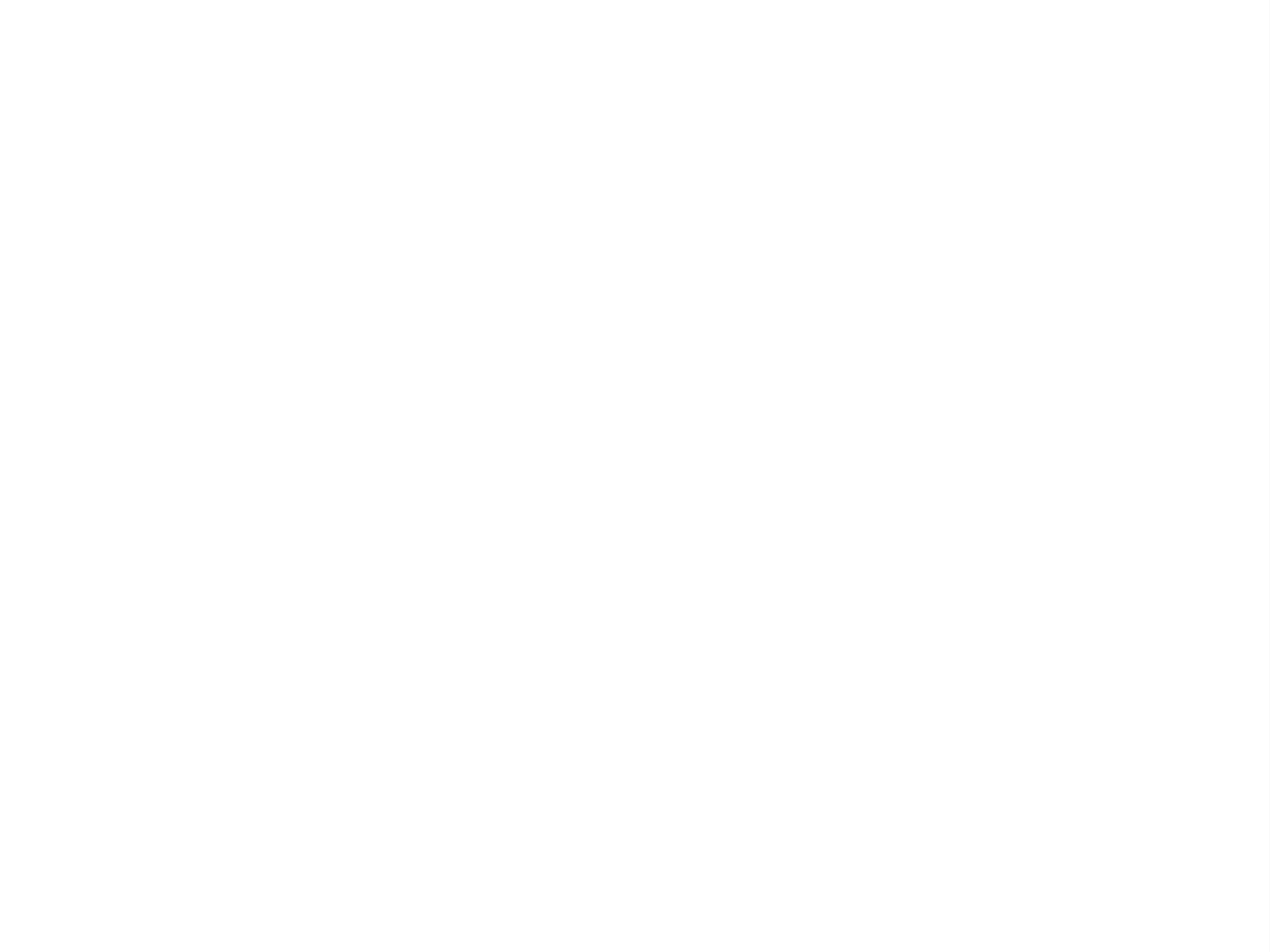}
 \end{center}  
 \noindent Suppose $e_1$ is Killer, $e_2$ 
 is Detective and $e_3$ is (ordinary) Civilian.   
 Their knowledge at the beginning   
 of the end game is as follows, as visualised in \fbox{A} 
 with agents' local scopes.  
\vspace{-0.1cm}
{\small 
  \begin{tcolorbox} [top=0.1mm,bottom=0.1mm]
    \textbf{Initial knowledge.}    
    Argument $a_1$: ``$e_1$ is Killer'', is in $e_1$'s scope. Argument $a_4$: ``$e_2$ is Detective'', is in $e_2$'s scope. Argument 
    $a_7$: ``$e_3$ is Civilian'', is in $e_3$'s 
    scope. By Detective's ability,  
    that ``$e_1$ is Killer'', and 
    that ``$e_3$ is Civilian'' are known to 
    $e_2$, which thus appear in $e_2$'s local scope.  
    Clearly, $e_2$ also knows that $a_1$ is known to $e_1$ and that
    $a_7$ is known to $e_3$. 
  \end{tcolorbox} 
  }  
\vspace{-0.1cm} 
\noindent Now, suppose a sequence of argumentations by them as follows. 
At each 
step, an agent publicly announces an argumentation (argument(s), 
attack(s)). Publicly announced arguments 
are coloured brighter in all figures.  
 We graphically represent $(a_1, a_2) \in R$  
 by $a_1 \rightarrow a_2$. 
\vspace{-0.1cm} 
 {\small 
 \begin{tcolorbox}[top=0.1mm,bottom=0.1mm]
 \begin{description}  
    \item[1. \normalfont \it $e_3$ says:] 
    ``$e_2$ is Killer'' (argument $a_9$). It is $e_3$'s guess, in mutual conflict with an alternative: 
    ``$e_1$ is Killer'' (argument $a_8$). See {\small \fbox{B}}.  
    \item[2. \normalfont \it $e_2$ says:]
      ``$e_2$ is Detective'' ($a_4$) as 
      a counter-argument to $a_9$, and  
      then that ``$e_1$ is Killer'' 
      ($a_5$). See {\small \fbox{C}}.  
     \item[3. \normalfont \it  $e_1$ responds:] 
         ``$e_1$ is Detective'' 
         \mbox{(argument $a_2$)}, 
         and  (i.e. due to ability 
         of Detective) that
         ``$e_2$ is Killer'' \mbox{(argument $a_3$)}, 
    as a counter-argument to $a_4$ and $a_5$. See 
{\small \fbox{D}}. $e_1$ is aware that $a_3$ is actually 
    in mutual conflict with $a_1$ as well as 
      that $a_2$ is attacked by $a_1$. 
     \item[4. \normalfont \it $e_2$ insists:] 
       ``$e_1$ is Killer'' 
         ($a_5$) 
         as a counter-argument to $a_3$ and $a_2$. See {\small \fbox{E}}. 
 \end{description} 
 \end{tcolorbox} 
 }
\vspace{-0.1cm} 
 
 \noindent \textbf{Local agent argumentations.}  Each agent 
  sees all publicly announced arguments 
  together with any other arguments 
  it knows \cite{Kakas05,Hadoux17}, thus, for {\small \fbox{E}}, 
  we have {\small \fbox{E1}}, {\small \fbox{E2}} and {\small \fbox{E3}} 
  as the local argumentations of 
  $e_1$, $e_2$ and respectively $e_3$. 
  \begin{center}
            \includegraphics[scale=0.11]{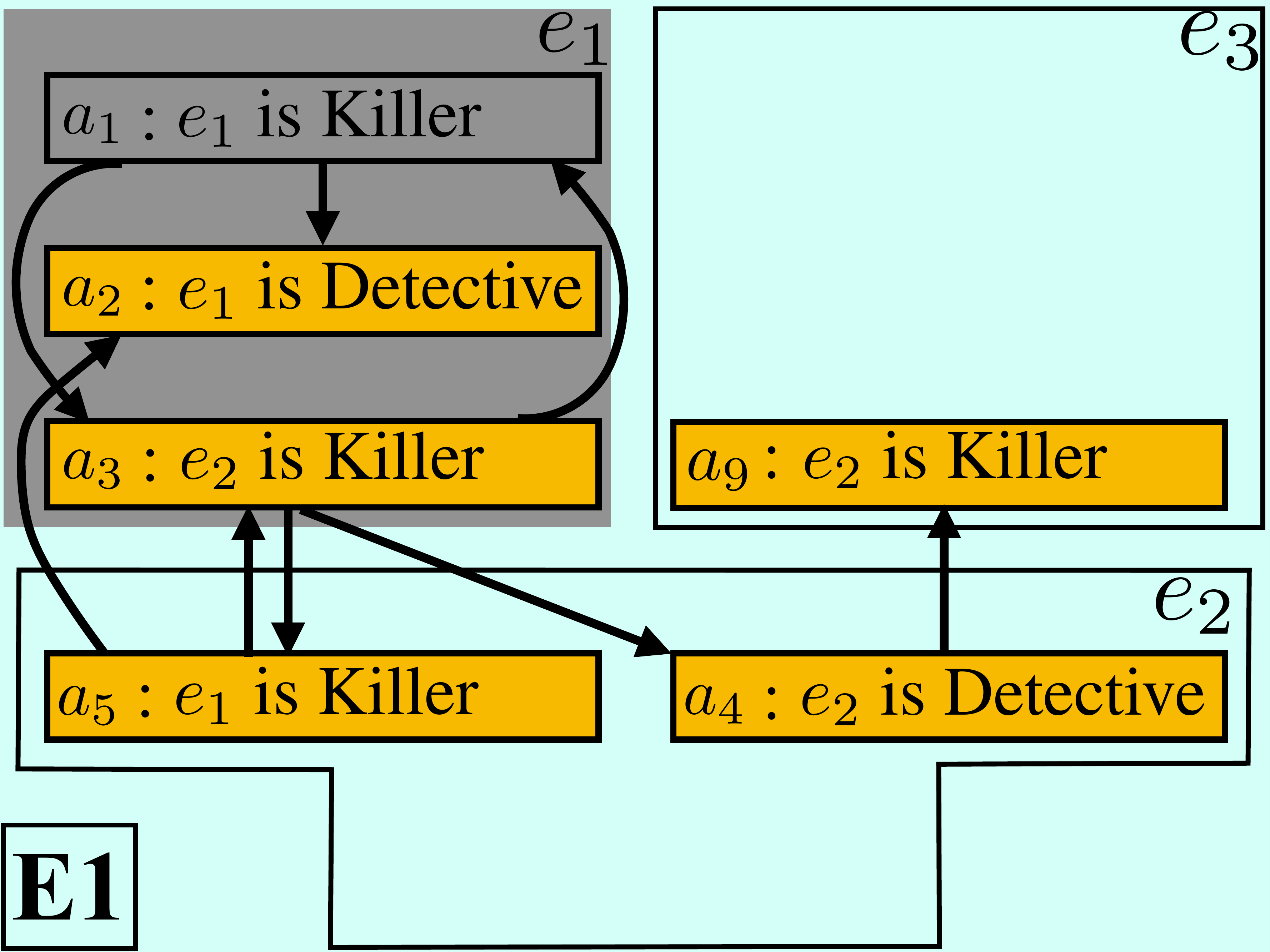} 
   \includegraphics[scale=0.11]{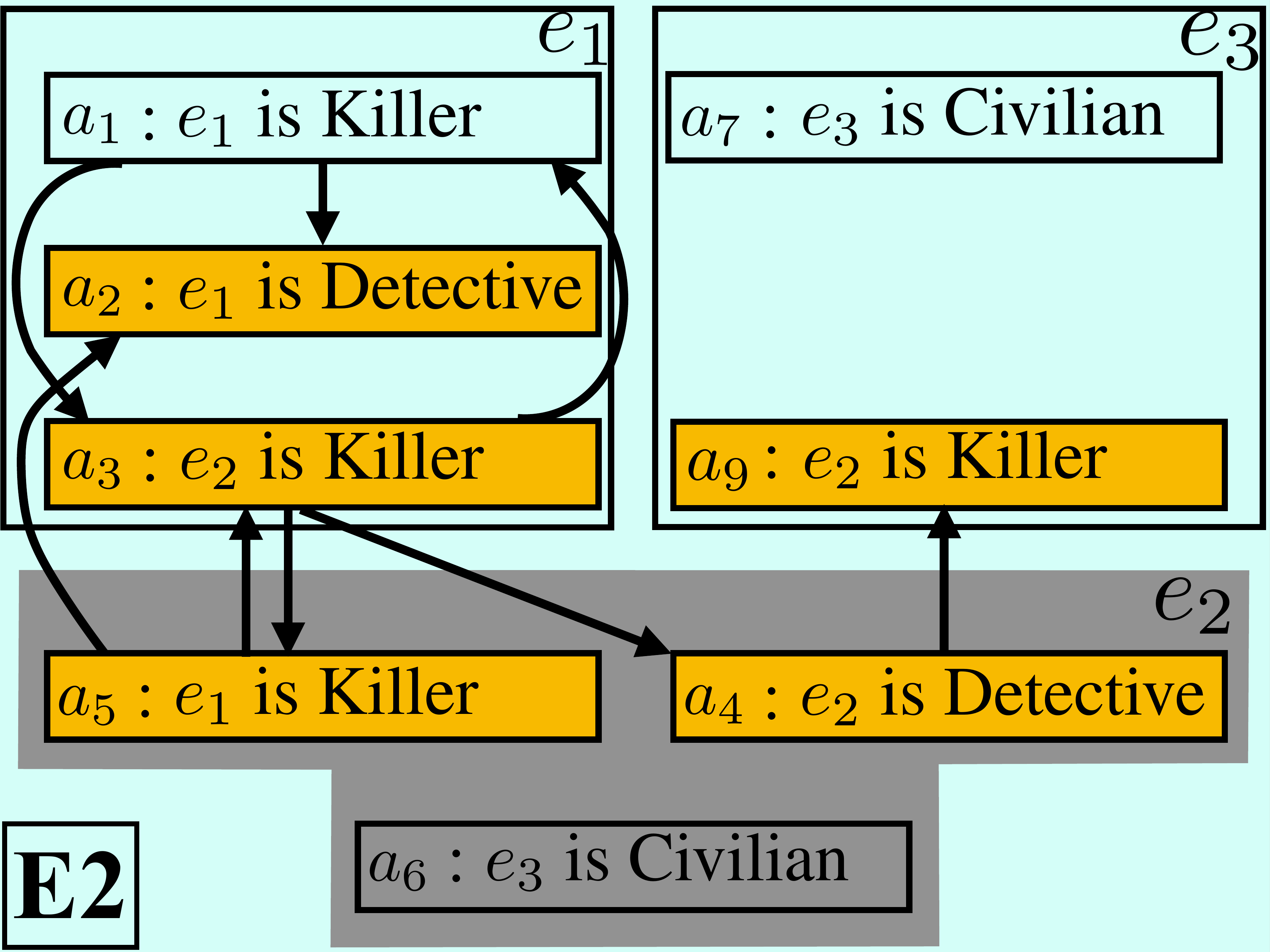}
     \includegraphics[scale=0.11]{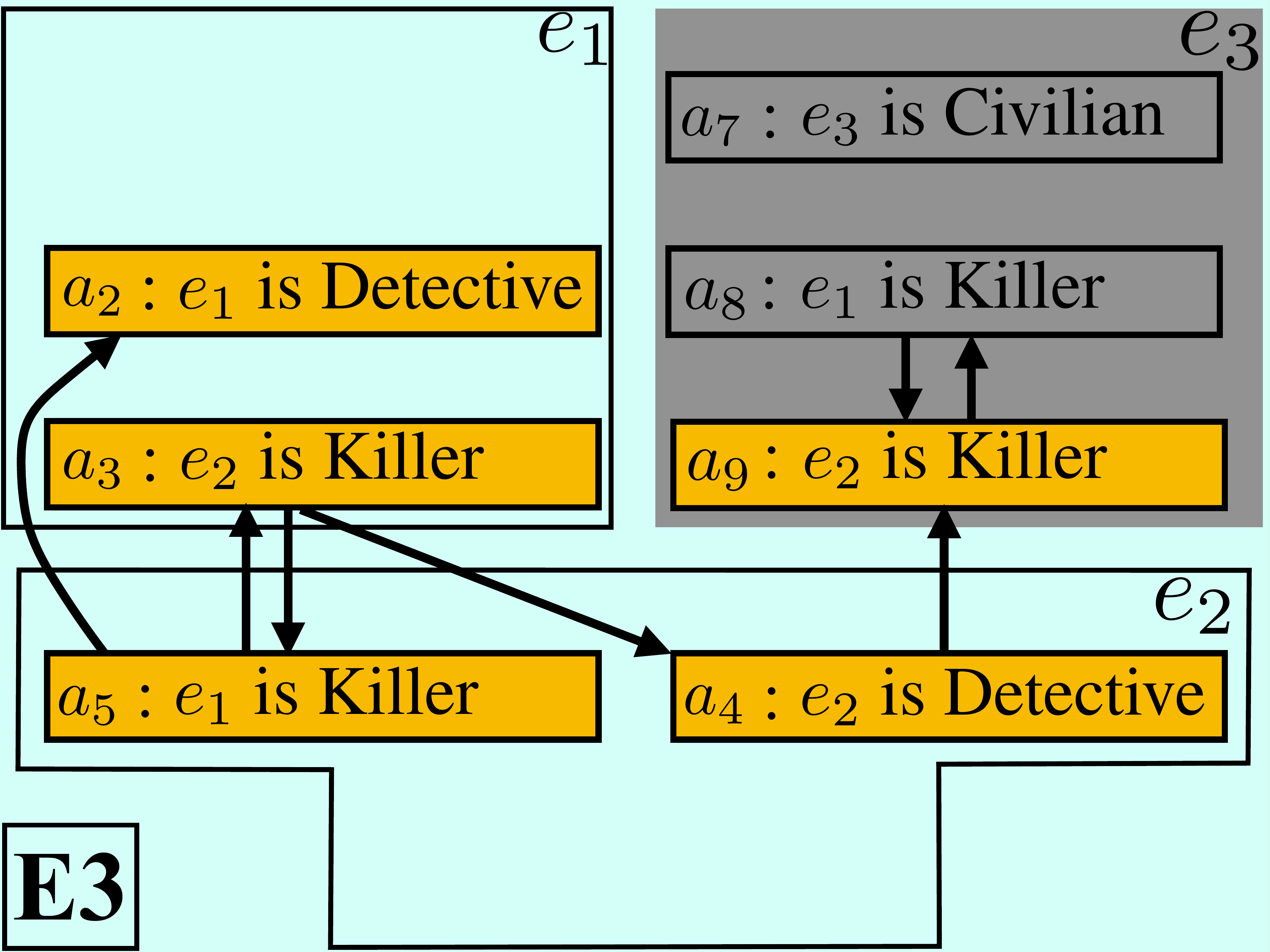} 
  \end{center}  
\subsection{Intra-agent preferences}\label{subsection_intra_agent_preferences}
   To talk of the role of intra-agent preferences, 
   suppose $e_1$ applies its own semantics $\semsmall$, 
   say $\pr$ (preferred semantics; see Section~2), to 
   the argumentation in {\small \fbox{E1}} to tell which arguments 
   are acceptable. By its definition (see Section~2), 
   $e_1$ considers either $\{a_1, a_4, a_5\}$ ($e_1$ is Killer
    and $e_2$ is Detective) 
   acceptable or else $\{a_2, a_3, a_9\}$ ($e_1$ is Detective 
    and $e_2$ is Killer) acceptable. 
   For 
    a rational judgement and not for a strategic purpose,
  however, 
    the second option is strange to say at the very least, 
    since it contradicts $e_1$'s factual knowledge $a_1$ ($e_1$
    is Killer). If we are to prioritise factual 
    arguments over the others, some attacks should 
    turn out to be spurious. 
    Similarly, $e_2$ who as Detective knows 
    $e_1$'s role should see  
    the attack of $a_3$ on $a_4$ as publicly announced by $e_1$ 
    is not factual: $a_4$ which $e_2$ knows factual 
    to it should refute $a_3$. 

    For fact-prioritised reasoning by an agent 
   of the argumentation it is aware of, we use an attack-reverse
     preference 
    per agent, 
    to prefer arguments that it knows factual (to 
    some agent) over the other arguments found 
    in its local agent argumentation. 
    {\small \fbox{E1}}, {\small \fbox{E2}} and {\small \fbox{E3}} with 
    preference-adjusted attack relations 
    are as shown in {\small \fbox{E1'}}, {\small \fbox{E2'}} and 
    {\small \fbox{E3'}}. 
     \begin{center} 
       \includegraphics[scale=0.11]{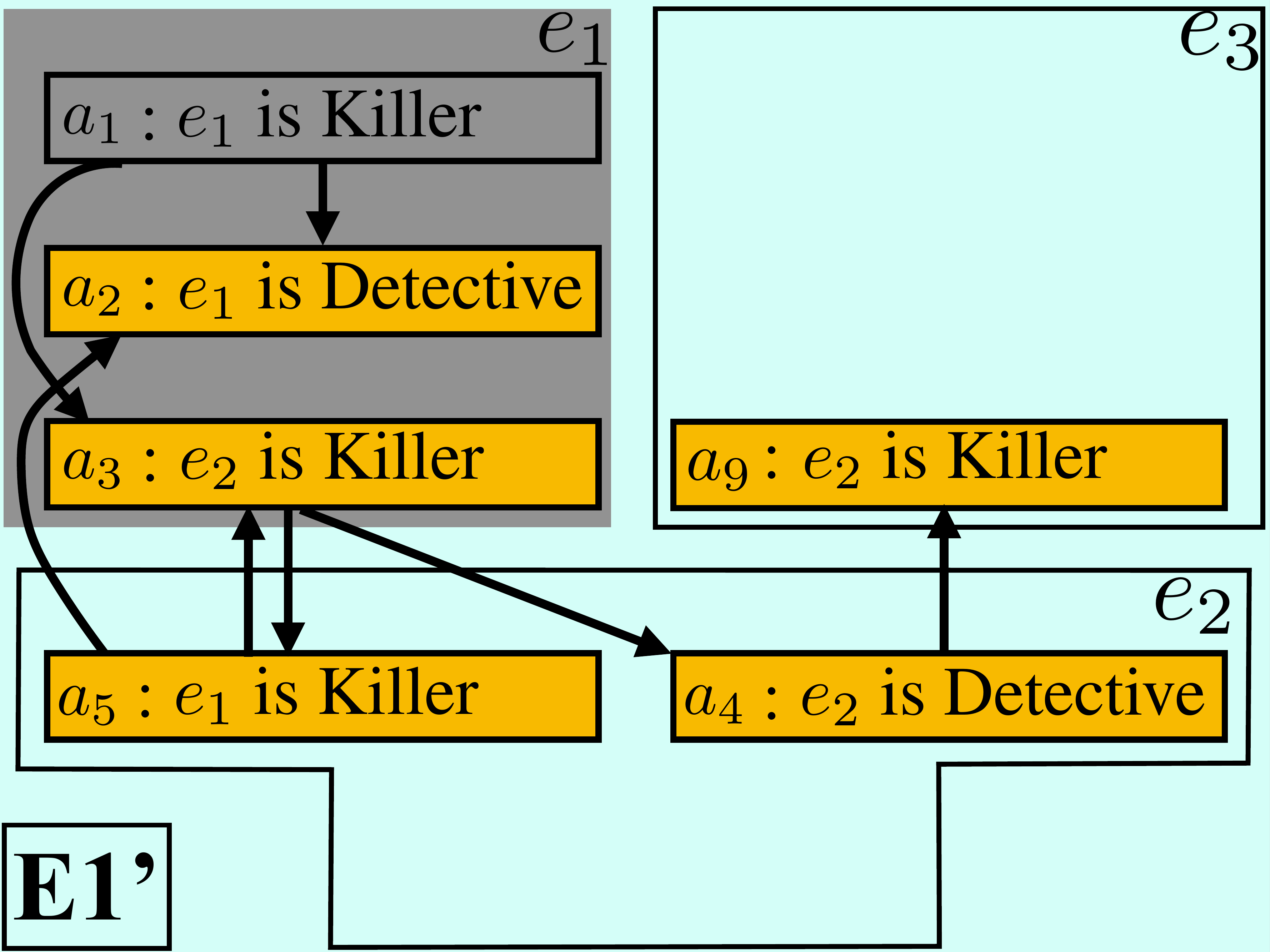}
       \includegraphics[scale=0.11]{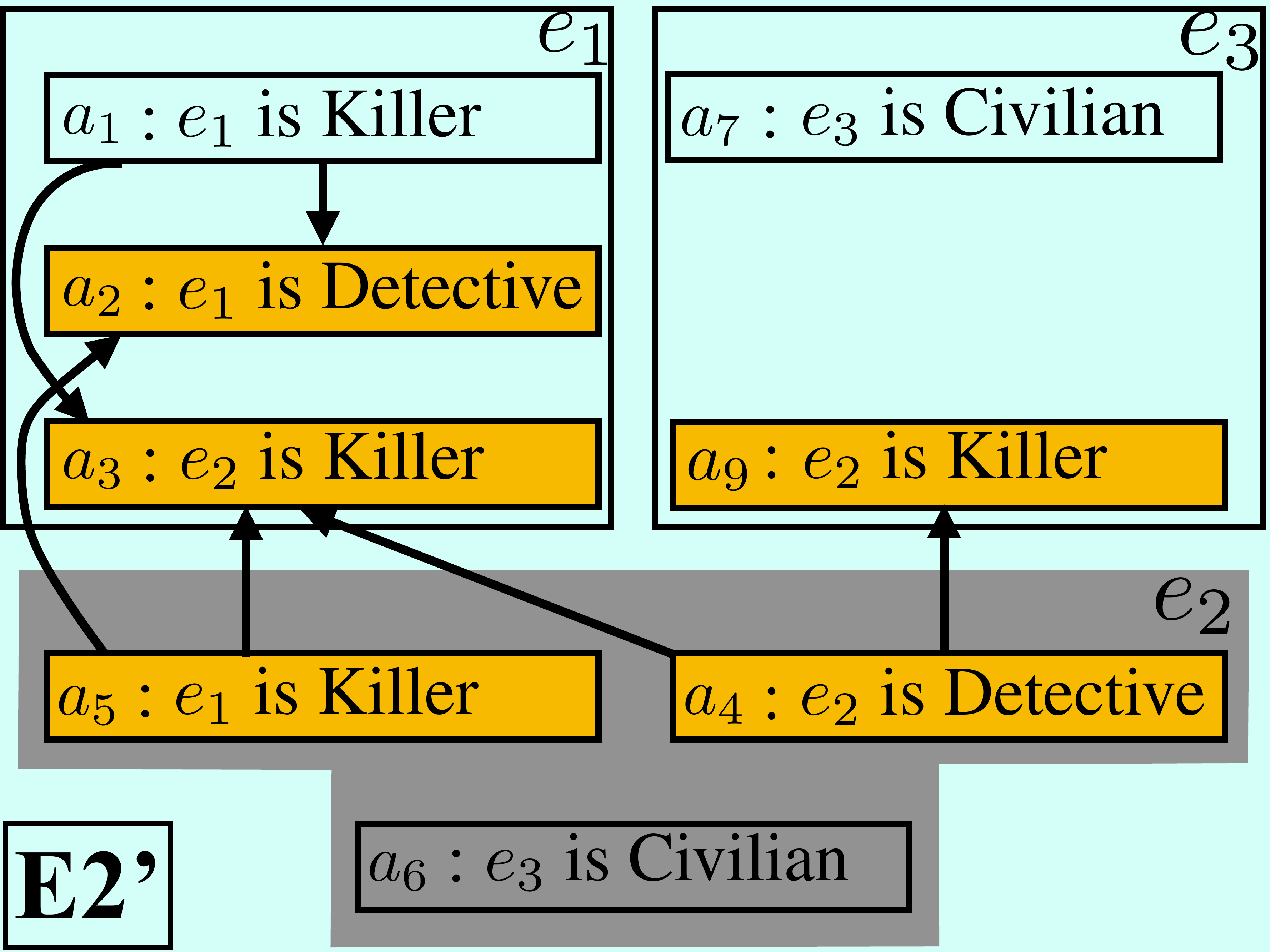}
       \includegraphics[scale=0.11]{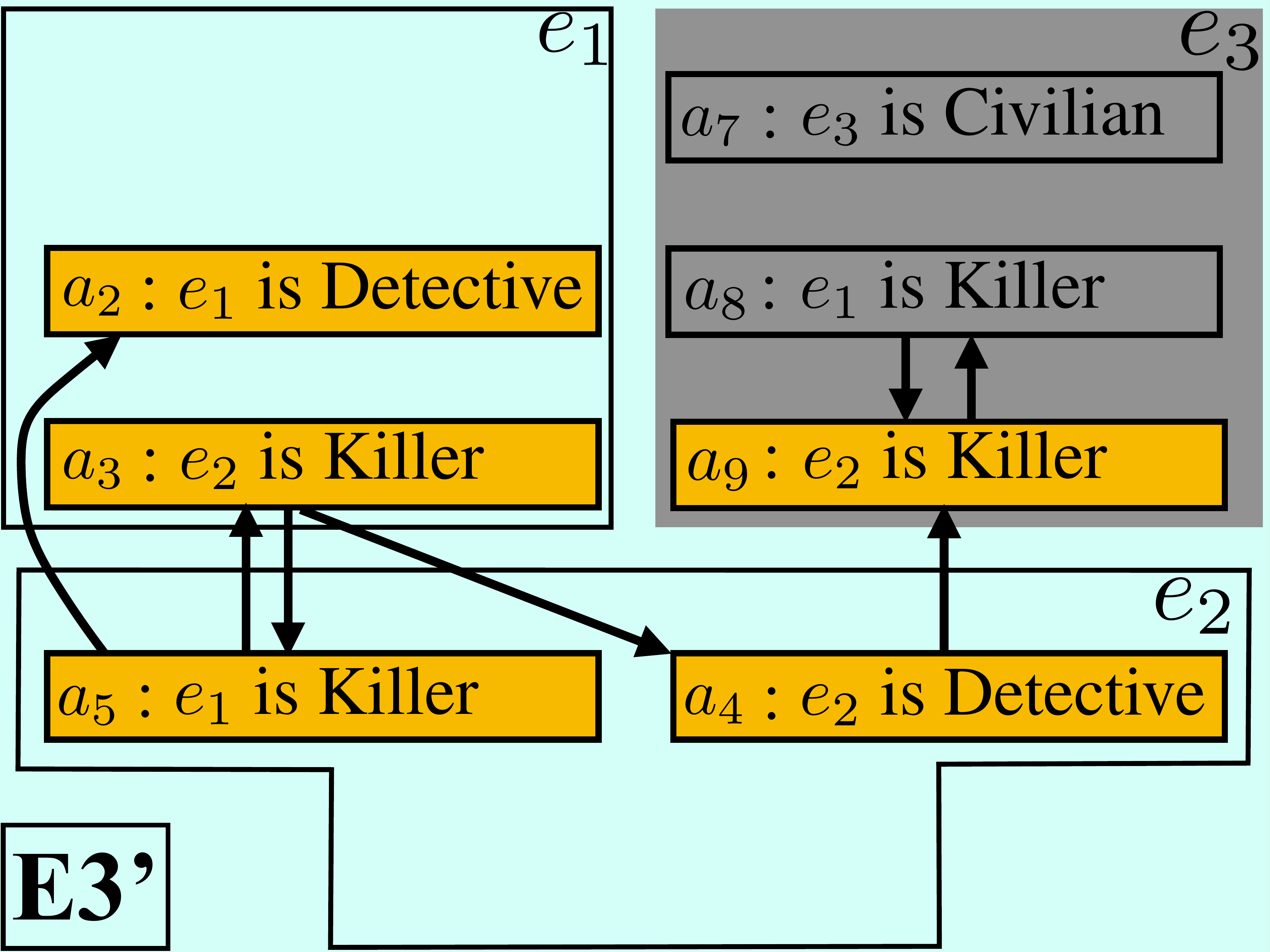}
   \end{center}   
Since both $e_1$ and $e_2$ know 
  that $e_1$ is Killer, i.e. $e_1$ knows $a_1$ to be factual to $e_1$,
  while  $e_2$ knows $a_1$ to be factual to $e_1$ and $a_5$ to be factual 
   to $e_2$,\footnote{$e_1$ cannot be certain $a_5$ is factual 
   to $e_2$, since, firstly, there may or may not be Detective 
   in a game, and, secondly, it could be Civilian who is bluffing 
   to be Detective.} 
   the attack from $a_3$ to $a_1$ is not in {\small \fbox{E1'}} or {\small \fbox{E2'}}. 
   Additionally, in {\small \fbox{E2'}}, the attack 
   from $a_3$ to $a_5$ is not present, and 
   the attack from $a_3$ to $a_4$ is reversed, 
   since $e_2$ knows $a_4$ and $a_5$ are factual to $e_2$.  
   By contrast, attacks in {\small \fbox{E3'}} 
   remain unchanged from {\small \fbox{E3}}, since 
   $e_3$ knows only that $a_7$ is factual to $e_3$.  
  \vspace{-0.2cm} 
  \subsubsection{On deception, and intra-agent preferences.}    
    A method of deception detection in two-party
   argumentation is found in 
   \cite{Sakama12}. In Section 5 of \cite{Sakama12} 
   that describes it, 
   an argument $a_x$ an agent $e_1$ puts forward 
   as an acceptable argument 
   is detected by an agent  
   $e_2$ to be deceptive if 
   $e_1$ has put forward an argument $a_y$ as acceptable  
   such that 
   {\small $a_x \underbrace{\rightarrow \cdots \rightarrow}_{2k+1} a_y$} 
   or {\small $a_y \underbrace{\rightarrow \cdots \rightarrow}_{2k+1} 
   a_x$} 
   for {\small $k \in \mathbb{N}$} (when there is 
   a graph path between $a_x$ and $a_y$ with an odd number of 
   edges), and that every argument in the path has been originally 
   put forward by $e_x$.

   In certain situations, 
   the proposed approach does not accurately model 
   deception detection, leading possibly to counter-intuitive results.  
   For example, consider $e_3$, Civilian, in our example. 
   As shown in {\small \fbox{B}}, $e_3$ chose to put forward 
   $a_9$ ($e_2$ is Killer) 
   as an acceptable argument. There, however, 
   was an alternative argument $a_8$ ($e_1$ is Killer) that could 
   have been put forward instead. These two arguments 
   are in mutual conflict, and only one of them may  
   be acceptable at one moment. But suppose, hearing the argumentation  
   by $e_2$ and $e_1$, that $e_3$ develops an impression 
   that $e_1$ is more likely the Killer, 
   since $a_4$ attacks $a_9$. Suppose $e_3$ then 
   changes its mind, and puts forward $a_8$ 
   as an acceptable argument, then $a_9$ becomes 
   non-acceptable. While, initially, $a_8$ was not considered 
   acceptable and $a_9$ acceptable (call it Scenario 1), and 
   later the acceptability statuses were swapped (call it Scenario 2),  
   the change was due to context change, 
   i.e. Scenario 1 seemed more likely to $e_3$ at the beginning of 
   the game, and Scenario 2 seemed more likely once the additional
   information was gained. For example once at {\small \fbox{C}}, $e_3$ could 
   have announced $a_8$ as acceptable, but that should not 
   lead to $e_3$'s deceptive intention in former announcement of 
   $a_9$. 
   The method in \cite{Sakama12} produces a false 
   positive in this kind of 
   a situation.  A false negative can also result. In our example, when 
   $e_1$ declares $e_2$ Killer (see {\small \fbox{D}}), 
   deceptive intention of $e_1$ should be already evident to $e_2$, 
   as it knows that ``$e_1$ is Killer'' is factual to $e_1$. However,  
   $e_1$ does not announce $a_1$ to obviously contradict itself  
   in public. But then the publicly known arguments   
   $a_2$ and $a_3$ do not attack each other, and 
   thus, according to the proposed approach, 
   $e_2$ will not detect $e_1$'s deception.   
   
    There is also an assumption on agents, 
   that they are attack-omniscient: if an agent learns 
   some arguments from another agent, it will 
   recover any attacks among them, whether or not  
   they were announced by that agent. In practice, it is not 
   necessary that 
   an agent is able to see an unannounced attack \cite{Takahashi16,Kakas05}; 
   however, more problematic to manipulable argumentation, 
  spurious attacks may be announced, 
   to complicate 
   deception detection in the absence of fact-prioritisation. 

   To see the point, suppose $a_x$ is an argument of $e_x$ that $e_x$ knows is 
   not acceptable in its local argumentation (by its $\semsmall$; 
   say $\pr$). 
   Suppose that $e_x$ nonetheless puts 
   $a_x$ forward. Suppose $e_x$ 
   needs later on to reveal $a_y$ as an acceptable 
   argument which $e_x$ knows attacks $a_x$. 
   Now, let us say that $e_y$ is $e_x$'s opponent. 
   If $e_x$ puts $a_y$ forward, $e_y$ will know 
   of $a_y$. Firstly, $e_y$ may not know $a_y$ attacks $a_x$ 
  \cite{Takahashi16,Kakas05}. However, even if $e_y$ sees the attack, 
   $e_x$ can safeguard against $e_y$'s reproach by 
   announcing a spurious attack from $a_x$ to $a_y$, 
   with which $e_x$ feigns context change (as we described earlier for 
   Civilian in our example) as an explanation for retracting 
   the previously announced acceptability status of $a_x$, i.e. 
   it concocts the following reasoning: (1) 
   $a_x$ and $a_y$ attack 
   each other; (2) only one of the two arguments is acceptable
   at one time; (3) but because both of 
   them may be acceptable, it was reasonable 
   that I (= $e_x$) previously 
   put $a_x$ forward as an acceptable argument; (4) 
   but now I am considering in 
   another context in which 
   $a_y$ instead is acceptable. \\
   \indent This way, 
   $e_x$ fakes the earlier described innocent belief change 
   as by Civilian.  
   With our example, 
   even if $e_1$ should announce 
   $a_1$ later, $a_1$ and $a_3$ are in mutual conflict 
   (and, even if only $a_1$ attacks $a_3$, the safeguarding 
   we have just described will produce the attack from $a_3$
    to $a_1$).
   Differentiation of the faking 
   from the innocent belief irresolution is not trivial if
   an agent cannot distinguish arguments in its local agent 
   argumentation. 
     \vspace{-0.3cm} 
   \subsubsection{Use of intra-agent preferences for deception/honesty detection.}   
   We address the difficulties above with the intra-agent preferences 
   to prioritise arguments that an agent knows are factual (to some agent); 
   see again {\small \fbox{E1'}} and {\small \fbox{E2'}}, where   
   $a_1$ attacks $a_3$ but not vice versa. For concrete steps to 
   detect deception/honesty, an agent should: 
      \begin{enumerate} 
         \item have the source argumentation (the one 
         with respect to which detection is conducted) 
  and the target argumentation (the one in which 
         deception/honesty may be detected). 
         \item calculate the semantics of the two argumentations. 
         \item restrict them to those arguments for which detection 
                is taking place. This restriction  
                is necessary since the two argumentations   
                may cover more arguments. It is also necessary
               to not restrict the two argumentations from a start 
               since the agent's rational judgement 
                as regards acceptability
             statuses of the concerned arguments 
                is based on them as a whole. 
         \item finally calculate the presence of     
        deception/honesty by applying an appropriate 
         criterion to compare the restricted semantics. 
      \end{enumerate}
   Let us first inspect deception detection 
   by considering 
   the transition from {\small \fbox{C}} to {\small \fbox{D}} 
  (re-listed below) induced by $e_1$'s public announcement. 
   \begin{center}  
        \includegraphics[scale=0.11]{WordsStep2justAttack.pdf}
        \includegraphics[scale=0.11]{WordsStep3justAttack.pdf}   
        \includegraphics[scale=0.11]{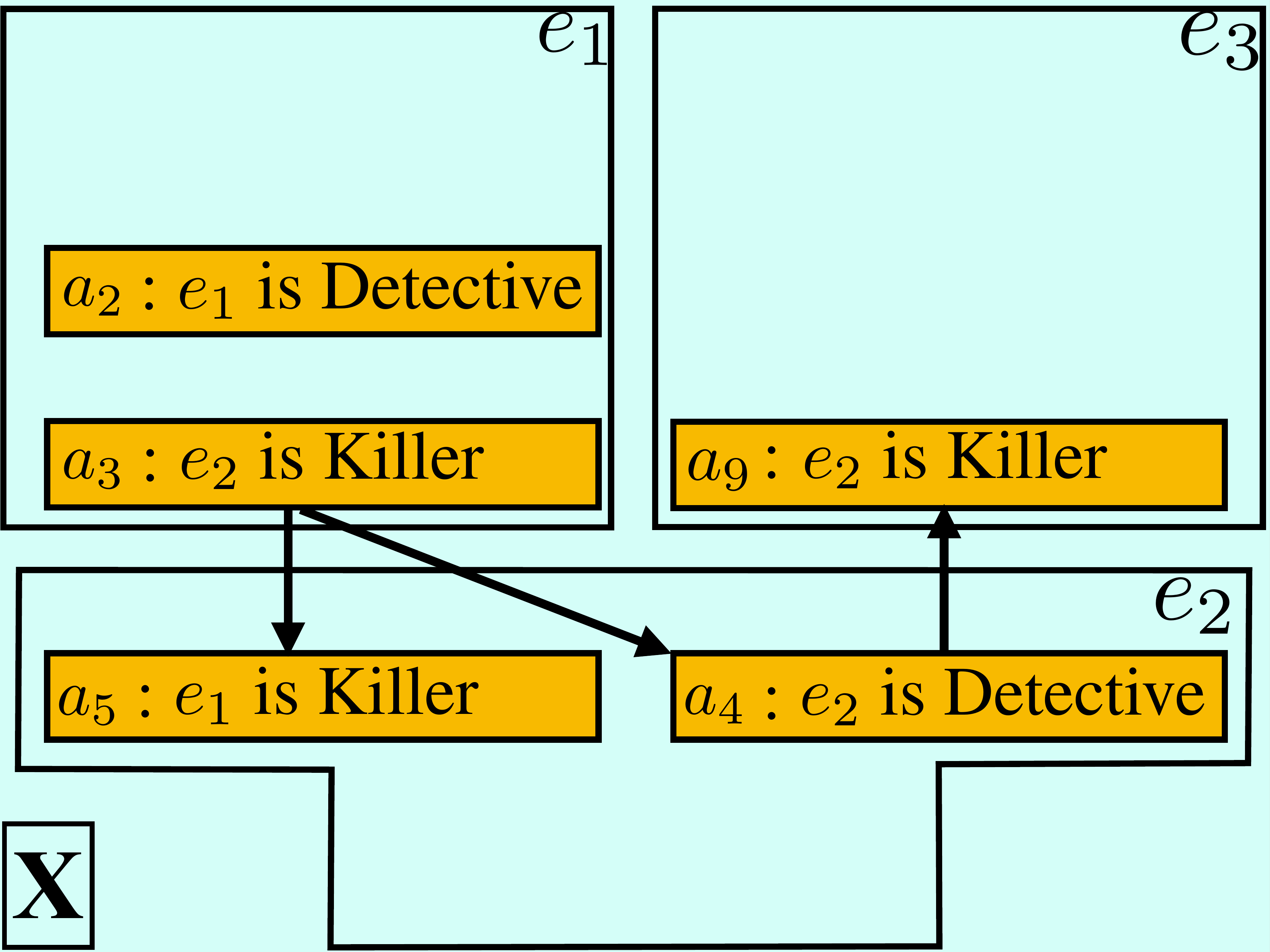} 
        \includegraphics[scale=0.11]{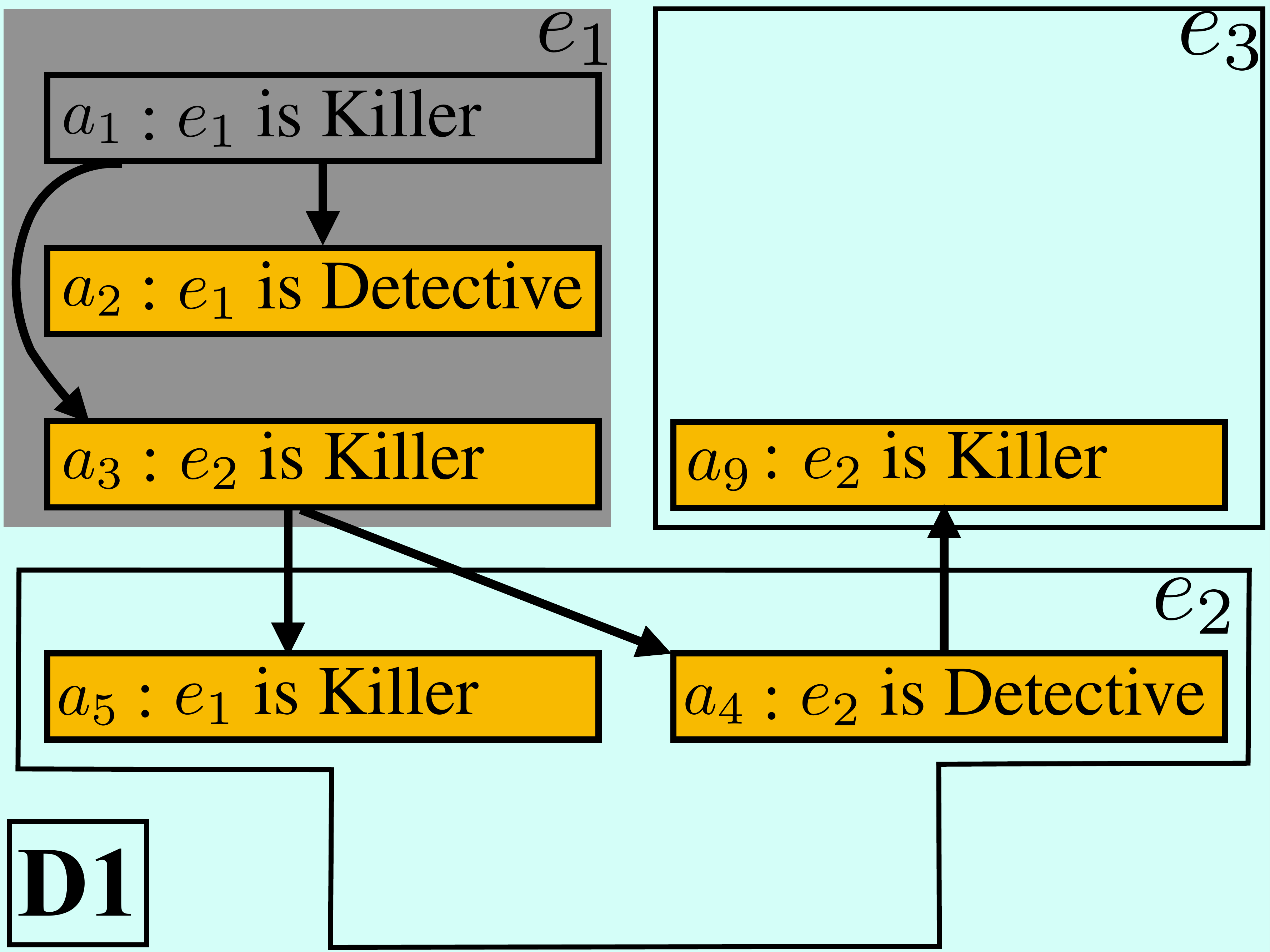}  
        \includegraphics[scale=0.11]{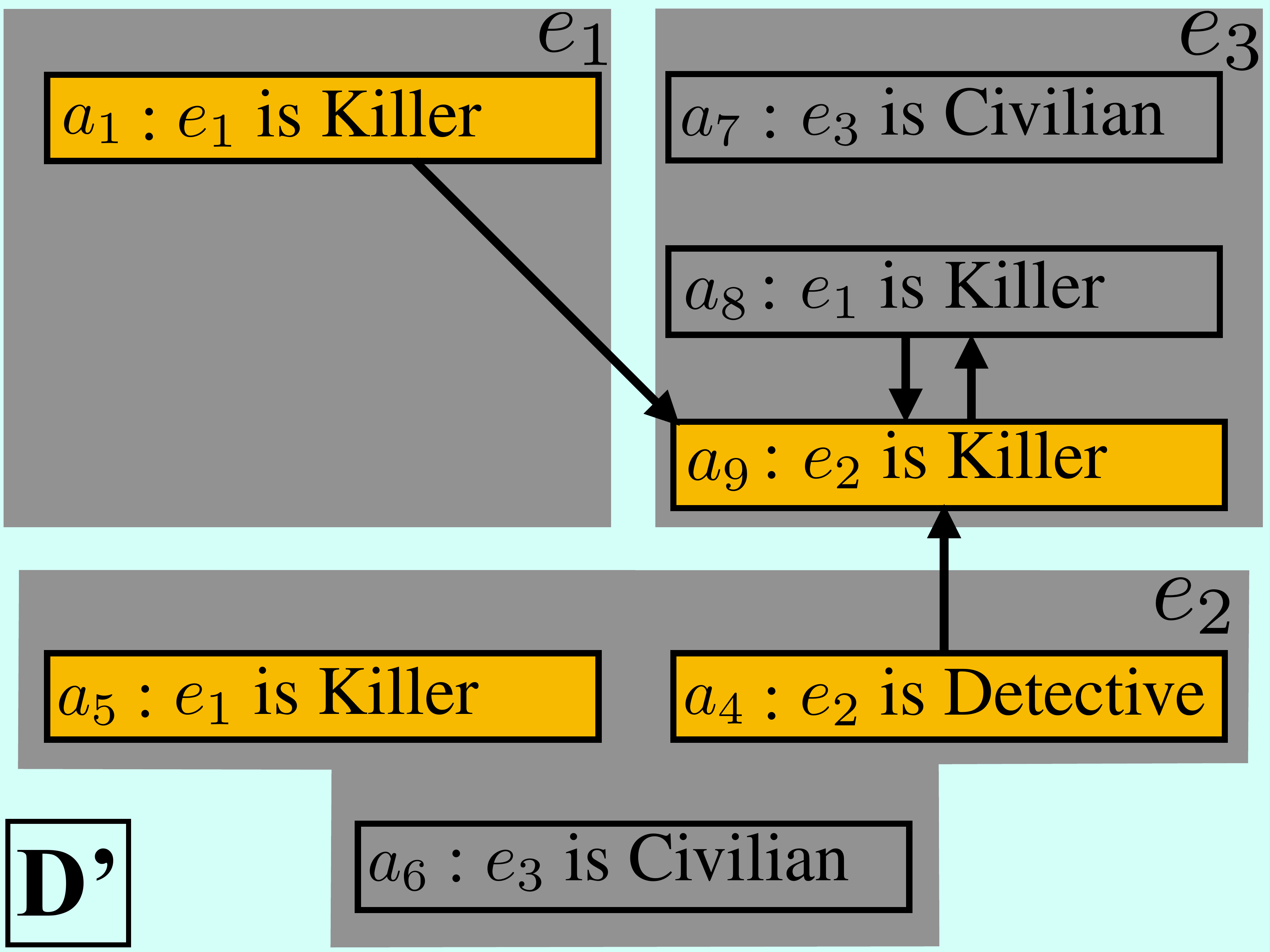}  
        \includegraphics[scale=0.11]{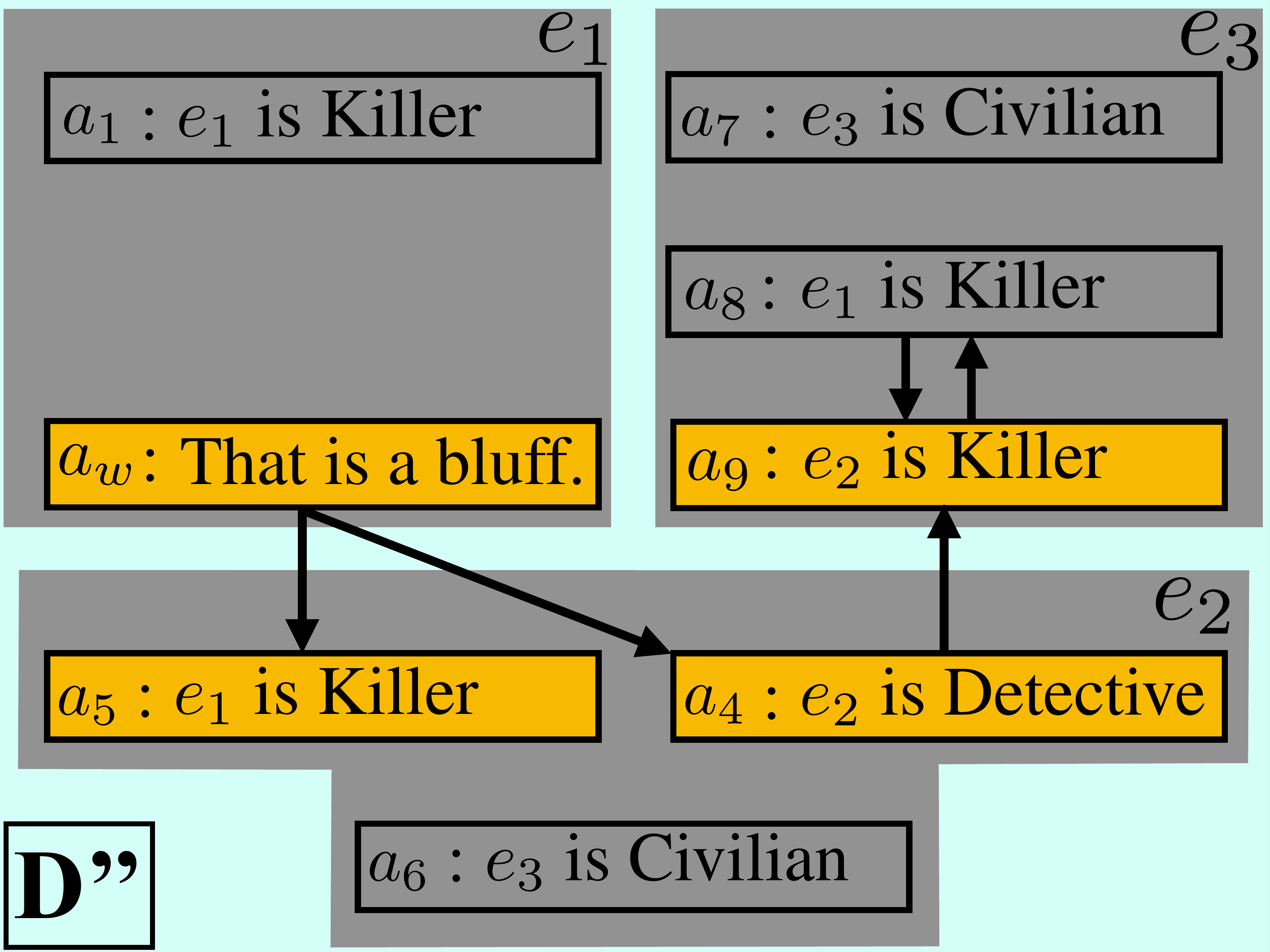} 
   \end{center} 
   Suppose it is $e_2$ that wants to check deception in the 
   new public announcement by $e_1$. \textbf{Step~1.} Since $e_2$ needs to 
   see any discrepancy between what $e_1$ has claimed 
   in public and what $e_2$ perceives $e_1$ actually thinks, 
   the source argumentation is the argumentation consisting only 
   of all the previous public announcements including $e_1$'s, 
   as shown in {\small \fbox{X}}, 
   while the target argumentation is $e_2$'s (opponent) model 
   of $e_1$'s local agent argumentation. For the 
   detection purpose, both must be already preference-adjusted 
   by what $e_2$ considers is $e_1$'s intra-agent preference, i.e.
   $e_2$'s model of $e_1$'s intra-agent preference. Here, 
   let us just assume that the source/target 
   argumentation is {\small \fbox{X}/\fbox{D1}}.\footnote{Recall
   $e_2$ knows $e_1$ knows $a_1$; as such, $a_1$ appears 
   in $e_2$'s model of $e_1$'s preference-adjusted 
    local agent argumentation. Recall also it is a common knowledge 
   that Killer does not know whether there be Detective; 
   as such, from $e_2$'s perspective, neither $a_4$ nor
    $a_5$ is known by $e_1$ to be factual to $e_2$.} 
    \textbf{Step 2.} The semantics of the source 
   argumentation is $\{\{a_2, a_3, a_9\}\}$, 
   and that for the target argumentation is $\{\{a_1,a_4,a_5\}\}$ 
   for a chosen $\semsmall \in \sem$.  
   \textbf{Step 3.} Note $e_2$ is checking the arguments in $e_1$'s 
   public announcement, which are $a_2$ and $a_3$. Hence, 
   the restriction of the semantics to them yield 
   $\{\{a_2, a_3\}\}$ (source) and $\{\emptyset\}$ (target). 
   \textbf{Step 4.} Recall that a semantics (= $\{A_1, \ldots, A_n\}$) 
   expresses non-deterministic
   possibilities, that each $A_i$ in the semantics 
   is judged possibly acceptable. 
   Thus, deception by $e_1$ is detected by $e_2$ certainly only when 
   the target semantics restricted to $a_2$ and $a_3$ (these 
   are what $e_2$ considers $e_1$ considers possibly acceptable) 
contains
   no member of the source semantics restricted to $a_2$ and $a_3$
   (these are what $e_2$ considers $e_1$ claims in public 
   to be possibly acceptable), 
   which holds good in this case because 
   $\{\emptyset\} \cap \{\{a_2, a_3\}\} = \emptyset$. 
    
   The differentiation of arguments 
   allows 
   us to also express detection of {\it honesty as
   truthfulness to arguments known to be factual}.  
   \footnote{This 
   differs from the honesty in \cite{Takahashi16}  
   to publicly announce only those arguments that the announcing 
   agent judges acceptable under the grounded semantics, according
   to which any agent who only announces its guesses that it accepts 
   becomes honest, which does not align well to our purpose.} 
  Suppose an alternative transition from {\small \fbox{C}} 
   with  $e_1$'s (rather silly) announcement 
   of $a_x:$ ``$e_1$ is Killer.'' into {\small \fbox{D'}}.  
   The first three steps of honesty detection 
   are the same as of deception detection.  
   Suppose $e_2$ is the detector, and suppose 
   $e_2$'s preference-adjusted model 
   of $e_1$ (the target argumentation) is 
   $(\{a_1, a_4, a_5, a_6, a_7, a_9\},\linebreak
   \{(a_1, a_9), (a_4, a_9)\})$, where $a_1$ is 
   considered factual to $e_1$. 
   The semantics of the source argumentation 
   is $\{\{a_1, a_4, a_5\}\}$, and that of 
   the target argumentation is 
   $\{\{a_1, a_4, a_5, a_6, a_7\}\}$. Since 
   $e_2$ is checking the argument in $e_1$'s public announcement, 
   they are restricted by $\{a_1\}$, yielding 
   $\{\{a_1\}\}$ (source) and $\{\{a_1\}\}$ (target).  
   \textbf{Step 4.} For detection of  
   honesty with respect to factual arguments, 
   each member $A_i$ of restricted source semantics
   $(\{A_1,\ldots, A_n\})$, 
   which $e_1$ has publicly claimed 
   acceptable, must consist only of the arguments 
   factual to $e_2$, since, if not, they can
   be just $e_1$'s guesses and bluffing to $e_2$. $A_i$ containing 
   any guesses is, insofar as 
   it is potentially deceptive, not certain honesty. 
   Moreover, the source 
   and the target semantics must exactly match; 
   in particular, the latter cannot contain strictly 
   greater a number of members than the source 
   argumentation\footnote{
   Since every public argumentation is known to every agent, 
   the converse is not possible.} which would imply 
   $e_1$'s withholding of factual information, 
   which again can be potentially a deceptive behaviour.  
   In this example ({\small \fbox{D'}}), $a_1$ is known to be factual to $e_2$, 
   and the two restricted semantics match exactly, so 
   $e_2$ detects $e_1$'s honesty. These two criteria 
   ensure that 
   $e_3$ does not detect $e_1$'s honesty at {\small \fbox{D'}}, 
   since $a_1$ is not known factual to $e_3$. 
\vspace{-0.3cm} 
  \subsection{Inter-agent preferences}\label{subsection_inter_agent_preferences}
\vspace{-0.1cm} 
      To talk of the role of inter-agent preferences, 
      let us say $e_3$ wants to decide which set(s) of arguments
      to publicly accept at \fbox{E} (to decide which agent
      should be hanged). $e_3$ then obtains 
      all the public argumentations announced up to \fbox{E} 
     (which is \fbox{X} plus two attacks from $a_5$ to $a_3$ 
    and from $a_5$ to $a_2$), as the basis of its reasoning. 
    It then adjusts it by its intra-agent preference, 
    to obtain its model of the public argumentation, 
    which in this particular example is again 
     \fbox{X} plus two attacks from $a_5$ to $a_3$ 
    and from $a_5$ to $a_2$, because 
    $e_3$ cannot tell whether any arguments by $e_1$ or $e_2$ 
    are factual. In the argumentation, 
      $e_3$ sees $a_3$ and $a_5$ in mutual
      conflict. With $\semsmall = \pr$, 
      $\{a_2, a_3, a_9\}$ ($e_2$ is Killer) and $\{a_4, a_5\}$ ($e_1$
     is Killer) 
       are two possible judgement. 
     
      Now, when $e_1$, $e_2$ and $e_3$ are 
      all strangers to each other, 
      it is likely that $e_3$ with $\semsmall = \pr$ will just have to 
      choose one of the two. 
      If, however, $e_3$ has gathered information  
      from previous interactions with them to the point where 
      $e_3$ considers 
      $e_2$ a liar and $e_1$ an honest agent, 
      then it is more likely that $e_3$ will trust
      $e_1$ more, to accept $\{a_2, a_3,a_9\}$ ($e_2$ 
      is Killer).   
      
      The trustworthiness of $e_1$ perceived 
      by $e_3$ can be expressed numerically.  
      Let $\mathbb{Z}$ be the class of all integers, 
      with a function $\ve: E \times E \rightarrow \mathbb{Z}$, 
      then  
      the numerical trust $e_3$ gives $e_1$ can be expressed 
      by $\ve(e_3, e_1)$. 
      Suppose also $\ve(e_3, e_2)$ such that    
      $\ve(e_3, e_2) < \ve(e_3, e_1)$.   
      By enforcing that a greater numerical value 
      implies a greater trust, we can express that $e_3$ trusts
      $e_1$ more, and can define an inter-agent preference 
      per agent to break mutually conflicting arguments in favour    
      of the agent(s) it trusts more.  
   \vspace{-0.4cm} 
\subsubsection{Deception and vagueness.}   
We assume that each agent updates trustworthiness of 
     those agents it interacts with when it detects 
     deception/honesty (see section 
     \ref{subsection_intra_agent_preferences}) 
     in the other agents. 
      As a result of this, 
      inter-agent preferences based on   
      the numerical trustworthiness 
      may dynamically update. 
      
      While not exactly in the scope of this paper, we briefly 
      describe how this may affect 
      agent's choice of which argumentation to publicly announce.  
      For a tactical advantage, 
      an agent may keep its public announcement vague, as often occurs 
      in politics\footnote{CMV: There should exist a system to ensure
      politicians admit to their blatant lies, at reddit.com.}, 
      instead of resorting to an obvious deception. As an example, 
      let us consider yet another transition
      from \fbox{C}, where $e_1$, instead of publicly announcing 
    a self-contradicting argumentation 
      with $a_2$ and $a_3$, which is easily detected by 
      Detective to be deceptive, opts for announcing 
      $a_w$, into {\small \fbox{D''}}, 
    to counter-attack $e_2$'s $a_4$ and $a_5$. Since 
      there is no possibility that Killer 
      would know of the presence of Detective, 
      it is always plausible to some degree that 
      $e_2$ is bluffing. Consequently, 
      $a_w$ cannot be contradicting arguments known 
      to $e_1$ to be factual in the local agent
      argumentation of any agent. With $a_w$, $e_1$ 
      maintains its perceived trustworthiness, which can be 
      a better tactic if this game is to be repeated multiple times. 
      In an extended work, we include experimental 
    results on this. 
\section{Manipulable Multi-Agent Argumentation 
     with Intra-/Inter-Agent Preferences}\label{section_manipulable_argumentation} 
\vspace{-0.2cm}     
We now formalise the intuition given in Section~\ref{section_motivation}. 
For the formalisation purpose, we consider the static part not involving 
public announcements a manipulable multi-agent argumentation, to 
which a public announcement provides dynamics.   
\begin{definition}[\MMA]\label{def_manipulable_multiagent_argumentation} 
{\small 
            Let $(\Fg, \Fpub, E, \fe, \fa, \gsem, \fleqintra, 
             \ve, \fleqinter)$ be a tuple with:  
    $\Fg \in \mathcal{F}^{\Dung}$; $\Fpub \in 2^{\Fg}$; 
    $E \subseteq \mathcal{E}$; 
 $\fe, \fa: E \rightarrow (2^{\Fg} \backslash (\emptyset, \emptyset))$; 
   $\gsem: E \times E \rightarrow \sem$; 
    $\fleqintra: 
     E \times E \rightarrow 2^{(A \cup \{\bot, \top\}) \times (A \cup \{\bot, \top\})}$;  
   $\ve: E \times E \rightarrow \mathbb{Z}$; and 
   $\fleqinter: E \rightarrow 2^{A \times A}$. 
    It is assumed that $\{\top, \bot\} \cap A = \emptyset$. 
    Such a tuple will be called a manipulable multi-agent 
   argumentation
     when it satisfies all the following: 
        \begin{enumerate}[leftmargin=0.3cm]  
            \item $\getR(\fe(e)) = \getR(\Fg) \cap (\getArg(\fe(e)) 
     \times \getArg(\fe(e)))$ for $e \in E$ $\andC$  \linebreak
       $\getArg(\fe(e_1)) \cap \getArg(\fe(e_2)) = \emptyset$ if 
      $e_1 \not= e_2$, $e_1, e_2 \in E$ 
 (\textbf{local scopes}).
   \begin{adjustwidth}{0.2cm}{} 
     {\normalfont\small Explanations: Global 
       argumentation, $\Fg$, is assumed to consist of 
       argumentations in each agent's local scope and
       attacks across any two of them. 
       As such, $\Fg$ restricted to $\getArg(\fe(e))$ 
       includes the same attacks 
       as in $\fe(e)$. No two local scopes overlap (see Section 
      \ref{subsection_agent_argumentation}).} 
   \end{adjustwidth} 
   \vspace{0.14cm} 
            \item $\fe(e) \in 2^{\fa(e)}$ for $e \in E$ 
                (\textbf{local agent argumentation}).  
      \begin{adjustwidth}{0.2cm}{} 
      {\normalfont\small Explanations: $e$'s local agent argumentation
    subsumes $e$'s local scope argumentation (Section~\ref{subsection_agent_argumentation}).}  
      \end{adjustwidth} 
    \vspace{0.14cm} 
           \item $\Fpub \in 2^{\fa(e)}$ for $e \in E$ 
     (\textbf{public subsumption}). 
    \begin{adjustwidth}{0.2cm}{}
    {\normalfont\small Explanations: 
     Every agent is aware of $\Fpub$ comprising the argumentation(s) 
   announced publicly. } 
   \end{adjustwidth} 
     \vspace{0.14cm} 
            \item $\fleqintra(e_1, e_2)$ for $e_1, e_2 \in E$ 
     is a partial order 
        on $\getArg(\fa(e_1)) \cup \{\bot, \top\}$ 
     (\textbf{partial order 1}).  
    \begin{adjustwidth}{0.2cm}{}
  {\normalfont\small Explanations: 
      Each $e_1 \in E$ considers intra-agent preference over 
   the members of $\getArg(\fa(e_1))$, i.e. those arguments $e_1$ 
    is aware 
  of.} 
   \end{adjustwidth} 
    \vspace{0.14cm} 
     \item $(\bot, \top) \in \fleqintra(e_1, e_2)$; $(\top, \bot) 
      \notin \fleqintra(e_1, e_2)$; 
             $(\bot, a) \in \fleqintra(e_1, e_2)$; and 
               $(a, \top) \in \fleqintra(e_1, e_2)$ for 
      $e_1, e_2 \in E$ and $a \in \getArg(\fa(e_1))$  
      (\textbf{top
     and bottom}). 
   \begin{adjustwidth}{0.2cm}{} 
{\normalfont\small Explanations:    
    For each intra-agent preference in $e_1$'s perspective, 
   any $a \in \getArg(\fa(e_1))$ is 
   at least as preferred as $\bot$, i.e. $\bot$ is the least 
   preferred in the set $(\getArg(\fa(e_1)) \cup \{\bot, \top\})$, 
   and, dually, 
     $a \in \getArg(\fa(e_1))$ is at most as preferred as 
   $\top$
     in  $(\getArg(\fa(e_1)) \cup \{\bot, \top\})$.}   
   \end{adjustwidth} 
   \vspace{0.14cm} 
            \item Either $(\top, a) \in \fleqintra(e_1, e_2)$,   
      or else $(a, \bot) \in \fleqintra(e_1, e_2)$ for 
        $e_1, e_2 \in E$ and $a \in \getArg(\fa(e_1))$ 
         (\textbf{binary}).  
   \begin{adjustwidth}{0.2cm}{} 
{\normalfont\small Explanations: 
      Every $a \in \getArg(\fa(e_1))$ is as preferable either 
   as $\bot$ or else as $\top$.} 
   \end{adjustwidth} 
     \vspace{0.14cm} 
            \item For  $e_1, e_2 \in E$ and 
           $a \in \getArg(\fe(e_2))$, 
       if $(\top, a) \in \fleqintra(e_1, e_1)$, then it holds that 
        $(\top, a) \in (\fleqintra(e_2, e_2) \cap  
         \fleqintra(e_1, e_2))$ 
   (\textbf{knowledge}). 
 \begin{adjustwidth}{0.2cm}{} 
 {\normalfont\small Explanations: If $e_1$ knows an argument $a_x$ in the scope 
   of another agent $e_2$ is factual to $e_2$, 
   then (1) $e_2$ knows $a_x$ is factual to $e_2$, 
   and (2) $e_1$ knows $e_2$ knows $a_x$ is factual to $e_2$.} 
  \end{adjustwidth} 
    \vspace{0.14cm} 
    \item $\fleqinter(e)$ for $e \in E$ is a partial order on $A$
    such that for $a_1, a_2 \in A$ and $e_1, e_2 \in E$:  
      $a_1 \in (\fe(e_1) \cap \fa(e))$ $\andC$ $a_2 
     \in (\fe(e_2) \cap \fa(e))$ $\andC$  
      $(a_1, a_2), (a_2, a_1) \in \Fpub$ $\andC$ 
      $(\top, a_1), (\top, a_2) \not\in \fleqintra(e, e)$ $\andC$ 
      $\ve(e, e_1) \leq \ve(e, e_2)$ iff 
       $(a_1, a_2) \in \fleqinter(e)$ 
    (\textbf{partial order 2}) 
  \begin{adjustwidth}{0.2cm}{}
  {\normalfont\small Explanations:  
    When $e$ wants to decide which
    set(s) of arguments in $\Fpub$ to publicly accept,
   and when there are mutually conflicting arguments 
   between two arguments, one in the scope of an agent, 
   and one in the scope of another agent, then      
   $e$ prefers the argument of the agent 
   it gives a greater trust to (via $\ve$), with an exception 
   that no factual arguments will become less preferable.} 
   \end{adjustwidth} 
     \end{enumerate} 
  We denote the class of all manipulable multi-agent argumentations
    by $\mathcal{F}^{MA}$, 
   and refer to each member of the class by $F^{MA}$
   with or without a subscript.   
}
    \end{definition}     
   Attacks in $\fa(e)$ 
    ($e$'s local argumentation) match exactly 
    those in $\Fg$ (global argumentation) for the arguments in 
     $\fe(e)$, i.e. local scopes 
     are faithfully reflected on $\Fg$, which signifies  
   that each agent is fully concious of its own local scope 
    argumentation. Proofs are in Appendix.
   \begin{proposition} 
         $(\getR(\fa(e)) \cup \getR(\Fg)) 
\cap (\getArg(\fe(e)) \times     
   \getArg(\fe(e))) =  \getR(\fe(e))$.  
   \end{proposition}    
  \noindent Also, there indeed exist manipulable 
    multi-agent argumentations: 
  \begin{theorem}[Existence] 
     $\mathcal{F}^{MA} \not= \emptyset$. 
  \end{theorem}
   Many of the components of a manipulable multi-agent argumentation
   are as defined in Section~2. 
   For $\Fpub \equiv (\Apub, \Rpub)$, 
    it is the `public' argumentation,
   comprising all the publicly announced argumentation(s). For $\gsem$, it generalises $\fsem$ by allowing 
   what an agent sees is another agent's semantics, e.g. 
    $\gsem(e_1, e_2)$ is $e_1$'s model of $e_2$'s semantics,  
    which may not be the same as $e_2$'s semantics, i.e. 
    possibly $\gsem(e_1, e_2) \not= \gsem(e_2, e_2)$. 
    For $\fleqintra$, it is the function for 
    intra-agent preferences. $\fleqintra(e, e)$ 
   for $e \in E$ is the intra-agent preference 
   that $e$ applies to $e$'s local agent argumentation. In general,   
    $\fleqintra(e_1, e_2)$ for $e_1, e_2 \in E$ 
    is $e_1$'s model of $e_2$'s intra-agent preference that $e_1$ applies 
   to $e_1$'s model of $e_2$'s local agent argumentation,
   whose formal definition is as follows:  

  \begin{definition}[Perceived argumentation]   
      Let $\om: \mathcal{F}^{MA} \times \mathcal{E}  \times \mathcal{E}
      \rightarrow 
           \mathcal{F}^{\Dung}$ be such that, 
     for any  $F^{MA} \equiv (\Fg, \Fpub,  E, \fe, 
        \fa, \gsem, \fleqintra, \ve, \fleqinter)$ and any $e_1, e_2 \in E$, 
      $\om(F^{MA}, e_1, e_2)$ 
  satisfies: $\Fpub \cup (\fa(e_1) \cap \fe(e_2)) \subseteq 
         \om(F^{MA}, e_1, e_2) \subseteq 
         \fa(e_1)$ (\textbf{epistemic bounds}). 
 We say that $\om(F^{MA}, e_1, e_2)$ is $e_1$'s 
    model of $e_2$'s argumentation, which 
    we also denote $e_1\smile e_2$.
    We assume $e_1 \smile e_2$ is exactly
     $\fa(e_1)$ if $e_1 = e_2$. 
  \end{definition}         
 By (\textbf{epistemic bounds}), $e_1 \smile e_2$ 
  is contained in the argumentation $e_1$ is aware of, 
  and contains the public argumentation as well as 
  arguments in $e_2$'s local scope $e_1$ is aware of.  
 Intra-agent-preference-adjusted $e_1 \smile e_2$ 
   results from applying $\fleqintra(e_1, e_2)$ to
   $e_1 \smile e_2$, which we use  
for deception/honesty detection (see 
   \ref{subsection_intra_agent_preferences}): 
  \begin{definition}[$\fleqintra$-adjusted $e_1 \smile e_2$]{\ }\\
   Let $\adjust: \mathcal{F}^{\Dung} \times  
        2^{(\mathcal{A} \cup \{\top, \bot\}) \times 
         (\mathcal{A} \cup \{\top, \bot\})} \rightarrow  
      \mathcal{F}^{\Dung}$ be such that  
          $\adjust((A_x, R_x), \leq_x) = 
        (A_x, R_y)$, where $R_y$ is $\leq_x$-adjusted $R_x$.\footnote{Cf.
     section~\ref{subsection_attack_reverse_preference} for preference-adjusted attack relation.} 
    For any  $F^{MA} \equiv (\Fg, \Fpub,  E, \fe, 
        \fa, \gsem, \fleqintra, \ve, \fleqinter)$ and 
    any $e_1, e_2 \in E$, we say that 
    $F^{MA}_x \in 2^{\om(F^{MA}, e_1, e_2)}$ is 
    preference-adjusted $e_1 \smile e_2$ 
    iff $F^{MA}_x = \adjust(\om(F^{MA}, e_1, e_2), \fleqintra(e_1, e_2))$.    We denote it by $e_1 \smile_{\fleqintra} e_2$. 
  \end{definition}  
\subsection{Acceptability semantics}      
   For each $e \in E$, there are several argumentations 
   in $\Fg$ to compute a semantics of, for (1) deception/honesty 
   detection (see \ref{subsection_intra_agent_preferences}) and for (2) determining which sets of arguments  
   $e$ should accept publicly (see \ref{subsection_inter_agent_preferences}). The main difference between (1) and (2) 
   is whether inter-agent preference is taken into account 
   in the computation of the semantics. Recall from Section~3 
   that inter-agent preferences are not based on 
   fact/non-fact distinction but on a subjective bias, 
   not suitable for detection of 
   deception/honesty for which elimination of
   as much bias is the key. On the other hand, 
   when $e$ sees mutually attacking arguments 
   among other agents that cannot be resolved 
   by $e$'s intra-agent preference, and when still
   $e$ needs to decide which sets of arguments
   to consider publicly acceptable, it is reasonable 
   that $e$ includes in the judgement 
   some empirical clue(s), trust in this paper, 
   that it has gained of 
   them. Thus, semantics can be 
   trust-neutral, for (1), and 
   trust-adjusted, for (2). 
  \begin{definition}[Trust-neutral semantics]\label{def_trust_neutral_semantics}{\ }\\
   For  
{\small $F^{MA} \equiv (\Fg, \Fpub, E, \fe, \fa, \gsem,  \fleqintra, \ve, \fleqinter)$}  
    and $e, e_2 \in E$, 
  we say $\Gamma \in 2^{2^{\mathcal{A}}}$ 
  is: $e$'s model of $e_2$'s  
   public semantics iff $\Gamma = 
   \Dung(\gsem(e, e_2), \adjust(\Fpub, \fleqintra(e, e_2)))$;
   and $e$'s model of $e_2$'s 
   local agent semantics iff $\Gamma = 
   \Dung(\gsem(e, e_2), e \smile_{\fleqintra} e_2)$. 
  \end{definition}  
  \begin{definition}[Trust-adjusted semantics]{\ }\\
       For  
$F^{MA} \equiv (\Fg, \Fpub, E, \fe, \fa, \gsem,  \fleqintra, \ve, \fleqinter)$  
    and $e \in E$, we say $\Gamma \in 2^{2^{\mathcal{A}}}$ 
    is $e$'s trust-adjusted public semantics iff 
    $\Gamma = \Dung(\gsem(e, e), \adjust(\adjust(\Fpub, \fleqintra(e, e)),
    \fleqinter(e)))$. 
  \end{definition} 
  
\begin{example}[$\mathcal{F}^{MA}$ and semantics] 
      Let us consider \fbox{D} and \fbox{E} from 
      Section~\ref{section_motivation}
      for illustration of these notations. 
      A manipulable multi-agent argumentation for  
      \fbox{D} is $F^{MA} \equiv (\Fg, \Fpub, E, \fe, \fa, \gsem, \fleqintra, \ve, \fleqinter)$ with:  
{\scriptsize  
      \begin{itemize}[leftmargin=0.3cm]  
          \item $\Fg \equiv (\{a_1, \ldots, a_9\}, 
           \{(a_1, a_3), (a_3, a_1), (a_1, a_2), 
                 (a_3, a_5), (a_3, a_4),
                 (a_4, a_9), (a_9, a_8), (a_8, a_9)\})$. 
          \item $\Fpub \equiv (\{a_9, a_2, \ldots, a_5\}, 
           \{(a_4, a_9), (a_3, a_4), (a_3, a_5)\})$. 
          \item $E \equiv \{e_1, e_2, e_3\}$. 
          \item $\fe(e_1) = \{a_1, a_2, a_3\}$. \quad\quad\ \ \ 
          $\fe(e_2) = \{a_4, a_5, a_6\}$.  \quad \quad \ \ 
          $\fe(e_3) = \{a_7, a_8, a_9\}$.     
          \item $\fa(e_1) = (\{a_1, \ldots, a_5, a_9\}, 
                   \{(a_1, a_3), (a_3, a_1), (a_1, a_2), 
                     (a_3, a_4), (a_3, a_5), (a_4, a_9)\})$.
          \item 
                $\fa(e_2) = (\{a_1, \ldots, a_7, a_9\},  
\{(a_1, a_3), (a_3, a_1), (a_1, a_2), 
                     (a_3, a_4), (a_3, a_5), (a_4, a_9)\})$.
        \item      $\fa(e_3) = (\{a_2, \ldots, a_5, a_7, \ldots, a_9\}, 
              \{(a_3, a_4), (a_3, a_5), (a_4, a_9)\})$. 
          \item $\gsem(e_i, e_j) = \semsmall_{ij}$ for $i, j
        \in \{1,2,3\}$. 
         \item $(\top, a_1), (a_x, \bot)
      \in \fleqintra(e_1, e_1)$ for each
      $a_x \in (\getArg(\fa(e_1)) \backslash 
       \{a_1\})$.  
          \item $(a_x, \bot) \in \fleqintra(e_1, e_y)$ 
        for each $a_x \in \getArg(\fa(e_1))$ and each $y \in \{2,3\}$.  
          \item $(\top, a_x), (a_y, \bot) 
      \in \fleqintra(e_2, e_2)$ for each 
            $a_x \in \{a_1, a_4,\ldots, a_7\} \equiv A_z$ and
           $a_y \in (\getArg(\fa(e_2)) \backslash A_z)$.   
          \item $(\top, a_1), (a_y, \bot) \in \fleqintra(e_2, e_1)$
        for each $a_y \in (\getArg(\fa(e_2)) \backslash \{a_1\})$. 
         \item  $(\top, a_7), (a_y, \bot) \in \fleqintra(e_2, e_3)$   
        for each $a_y \in (\getArg(\fa(e_2)) \backslash \{a_7\})$.  
         \item $(\top, a_7), (a_y, \bot) \in \fleqintra(e_3, e_3)$
        for each $a_y \in (\getArg(\fa(e_3)) \backslash \{a_7\})$. 
         \item $(a_y, \bot) \in \fleqintra(e_3, e_w)$
     for each $a_y \in \getArg(\fa(e_3))$ and each $w \in \{1,2\}$.  
          \item $\ve(e_i, e_j) = n_{ij}$ for each
           $i, j \in \{1,2,3\}$.  
      \end{itemize}  
}    
  \noindent That for \fbox{E} is $F^{MA}_x \equiv 
   ({\Fg}_x, {\Fpub}_x, E_x, {\fe}_x, {\fa}_x,{\gsem}_x, 
    {\fleqintra}_x, {\ve}_x, {\fleqinter}_x)$ with: 
   $K_x = (\getArg(K) \cup \{a_5\}, \getR(K) \cup \{(a_5, a_2), (a_5, a_3)\})$ 
   for $K \in \{\Fg, \Fpub, \fa(e_1), \fa(e_2), \fa(e_3)\}$; 
   $\fe = {\fe}_x$; ${\gsem}_x(e_i, e_j) = 
    {\semsmall_x}_{ij}$ for $i, j \in \{1,2,3\}$;  
   $\fleqintra = {\fleqintra}_x$; and 
   ${\ve}_x(e_i, e_j) = {n_x}_{ij}$ for each $i,j \in \{1,2,3\}$. 
   
  \indent For $F^{MA}$, 
  {\small $e_2 \smile_{\fleqintra} e_1 = (\{a_1, \ldots, a_5, a_9\}, 
         \{(a_1, a_2), (a_1, a_3), (a_3, a_4), (a_3, a_5), 
          (a_4, a_9)\})$}, as in \fbox{D1}, 
while  
  {\small $e_2 \smile_{\fleqintra} e_2 = (\{a_1, \ldots, a_7, a_9\},
        \{(a_1, a_2), (a_1, a_3), (a_4, a_3), (a_5, a_3), (a_4, a_9)\})$}. 
     Also, $e_2$'s 
    $\fleqintra(e_2, e_1)$-adjusted model of $\Fpub$ is
   {\small $(\{a_2, \ldots, a_5, a_9\}, \{(a_3, a_4), (a_3, a_5), (a_4, a_9)\})$}, 
   denote it by 
  $F^{\Dung}_{\alpha}$. 
   
   Hence, $e_2$'s model of $e_1$'s public semantics 
   is $\Dung(\semsmall_{21}, F^{\Dung}_{\alpha}) = \{\{a_2, a_3, a_9\}\}$,
   and $e_2$'s model of $e_1$'s local agent semantics 
   is $\Dung(\semsmall_{21}, e_2 \smile_{\fleqintra} e_1)
   = \{\{a_1, a_4, a_5\}\}$. 
   These are trust-neutral semantics.  \\ 
   \indent For trust-adjusted semantics,  
   let us look at $e_3$ in $F^{MA}_x$. 
   $e_3$'s $\fleqintra(e_3, e_3)$-adjusted model of
   $\Fpub$ is 
   $(\{a_2, \ldots, a_5, a_7, \ldots, a_9\}, 
    \{(a_3, a_4), (a_3, a_5), (a_5, a_2), (a_5, a_3), (a_4, a_9)\})$, 
   denote it $F^{\Dung}_{\beta}$. 
   In $F^{\Dung}_{\beta}$ in which 
   fact/non-fact distinction is already taken into account, 
   still $a_3$ in $e_1$'s local scope and  
   $a_5$ in $e_2$'s local scope are in mutual conflict.  
   As per (\textbf{partial order 2}), 
   if, here, ${{\ve}_x}(e_3, e_1) < {{\ve}_x}(e_3, e_2)$, 
   then we have $(a_3, a_5) \in {\fleqinter}_x(e_3)$, 
   and so $e_3$'s trust-adjusted public semantics 
   is $\Dung({\semsmall_x}_{33}, \adjust(F_{\beta}, 
         {\fleqinter}_x(e_3))) = \{\{a_4, a_5\}\}$. 
   If, on the other hand,  ${{\ve}_x}(e_3, e_2) < {{\ve}_x}(e_3, e_1)$, 
   then we have $(a_5, a_3) \in {\fleqinter}_x(e_3)$, 
   and so $e_3$'s trust-adjusted public semantics 
   is $\Dung({{\semsmall}_x}_{33}, \adjust(F_{\beta}, 
         {\fleqinter}_x(e_3))) = \{\{a_2, a_3, a_9\}\}$.  

  \end{example}

  \subsection{Updates: public announcements}   
   Since we allow public announcements by agents, we need to express 
   an update on a $F^{MA} \in \mathcal{F}^{MA}$.  
   In this work, we will assume for each $e \in E$ that $\fa(e)$ 
   monotonically increases only with publicly announced argumentations. 
   More general updates are left to future work.   
   As per our discussion in Section~3, 
   an agent may announce an attack of an argument on (or from) an argument 
   that has been 
   already announced by another agent, e.g. $a_4 \rightarrow$   
   from \fbox{B} to \fbox{C} or $a_5 \rightarrow$ from 
   \fbox{D} to \fbox{E}, where the target (or the source) 
   is missing since no agent announces arguments in 
   others' local scopes. (It is possible that 
   an announcement is done by two agents simultaneously, in which
   case it can include arguments in two distinct local scopes.) 
   Consequently, 
   we need to view a public announcement as a pre-Dung argumentation:
   \begin{definition}[Pre-Dung argumentations]   
      For any $A \subseteq_{\text{fin}} \mathcal{A}$ and any 
     $R \subseteq_{\text{fin}} \mathcal{A} \times \mathcal{A}$, 
   we say that $(A, R)$ is a pre-Dung argumentation iff 
    $a_1 \in A$ $\orC$ 
    $a_2 \in A$  for every $(a_1, a_2) \in R$. We denote the class of all 
    pre-Dung argumentations by $\mathcal{F}^{p\Dung}$, 
    and refer to each member of $\mathcal{F}^{p\Dung}$ 
    by $F^{p\Dung}$ with or without a subscript. 
   \end{definition} 
   Trivially, $\mathcal{F}^{\Dung} \subseteq 
    \mathcal{F}^{p\Dung}$. 
We consider an update as a function 
   $\update: \mathcal{F}^{MA} \times \mathcal{F}^{p\Dung} \rightarrow 
         \mathcal{F}^{MA}$, but 
as a composition of 
   a public announcement $\announce: \mathcal{F}^{MA} 
   \times \mathcal{F}^{p\Dung} \rightarrow 
  \mathcal{F}^{MA} \times \mathcal{F}^{p\Dung} \times 
   \mathcal{F}^{MA}$ and 
 agents' revision of agent-to-agent 
  trusts (given by $\ve$) in response to 
   detected deception/honesty, if any,  
  which we express in $\revV: \mathcal{F}^{MA} \times 
   \mathcal{F}^{p\Dung} \times \mathcal{F}^{MA} 
   \rightarrow \mathcal{F}^{MA}$. That is, 
   $\update = (\revV \circ \announce)$. We first 
   define $\announce$, responsible for
   expanding a manipulable multi-agent argumentation.  
   In the rest, we assume for 
$F^{p\Dung}_1, F^{p\Dung}_2 \in \mathcal{F}^{p\Dung}$  and 
    $\oplus \in \{\cup, \cap\}$ 
   that 
   $F^{p\Dung}_1 \oplus F^{p\Dung}_2 \equiv (A_x, R_x)$
with: $A_x \equiv (\getArg(F^{p\Dung}_1) \oplus \getArg(F^{p\Dung}_2))$; 
  and $R_x \equiv (\getR(F^{p\Dung}_1) \oplus \getR(F^{p\Dung}_2)) \cap 
    (A_x \times A_x)$. 
 \begin{definition}[Public announcements]\label{def_public_announcements}{\ }\\       
      Let $\announce: \mathcal{F}^{MA} \times \mathcal{F}^{p\Dung} 
    \rightarrow \mathcal{F}^{MA} \times 
      \mathcal{F}^{p\Dung} \times \mathcal{F}^{MA}$ be such that  
      $\announce(F^{MA}_1, F^{p\Dung})$ for 
     $F^{MA}_1 \equiv ({\Fg}_1, {\Fpub}_1, E_1, {\fe}_1, 
   {\fa}_1, {\gsem}_1,
        {\fleqintra}_1, {\ve}_1, {\fleqinter}_1)$ and $F^{p\Dung} \in \mathcal{F}^{p\Dung}$
    is defined iff, for every $a_1 \in \getArg(F^{p\Dung})$ and
         $a_2 \in \mathcal{A}$, 
 
    \begin{itemize}     
      \item If $(a_1, a_2) \in \getR(F^{p\Dung})$ $\orC$ 
         $(a_2, a_1) \in \getR(F^{p\Dung})$, then $a_2 \in \getArg(F^{p\Dung} \cup \Fpub)$ \linebreak
    (\textbf{no leak}). 
      \begin{adjustwidth}{0.2cm}{} {\normalfont\small Explanations:    
        Even if $F^{p\Dung}$ is not a Dung argumentation,  
        $F^{p\Dung} \cup \Fpub$ is.}  
      \end{adjustwidth} 
      \item  If $a_1 \in \getArg(\Fpub)$, then for every 
      $a_2 \in \getArg(F^{p\Dung} \cap \Fpub)$, 
      if 
          $(a_i, a_j) \in \getR(F^{p\Dung})$ with: $i, j \in 
    \{1,2\}$; and $i\not=j$, then 
         $(a_i, a_j) \not\in \getR(\Fpub)$ 
             (\textbf{no 
             repetition}).  
     \begin{adjustwidth}{0.2cm}{}
   {\normalfont\small Explanations:   
            Exactly the same 
            argumentation will not be announced again. Thus, 
            if an argument that has been already announced 
            is announced, attack(s) must be new.}  
    \end{adjustwidth} 
       \end{itemize} 
  \indent In case it is defined, let $\announce(F^{MA}_1, F^{p\Dung})$ be 
   $(F^{MA}_a, F^{p\Dung}_b, F^{MA}_2)$ with 
$F^{MA}_2 \equiv ({\Fg}_2, {\Fpub}_2,\linebreak E_2, {\fe}_2, {\fa}_2,
       {\gsem}_2, {\fleqintra}_2, {\ve}_2, {\fleqinter}_2)$  such that
      $F^{MA}_1 = F^{MA}_a$, that 
      $F^{p\Dung} = F^{p\Dung}_b$, and 
      that $F^{MA}_2$ satisfies all the following conditions. 
     \begin{enumerate} 
         \item {\small ${\Fg}_2 = {\Fg}_1 \cup F^{p\Dung}$ $\andC$  
    ${\Fpub}_2 = {\Fpub}_1 \cup F^{p\Dung}$ $\andC$     
   ${\fa}_2(e) = {\fa}_1(e) \cup F^{p\Dung}$, $e \in E$} 
      (\textbf{expansion}).  
     \vspace{0.14cm} 
        \item $E_2 = E_1$ (\textbf{same agents}). {\normalfont\small Explanations: 
           No new agent will join in.} 
    \vspace{0.14cm} 
         \item  ${\gsem}_1 = {\gsem}_2$ 
         (\textbf{same semantics}). {\normalfont\small Explanations: No agent 
            $e \in E_1$ 
            will change $\gsem(e, e_x)$ for any $e_x \in E_1$.} 
    \vspace{0.14cm} 
        \item ${\fe}_1(e) \subseteq {\fe}_2(e)$ for $e \in E_1$
     (\textbf{local scope monotonicity}). {\normalfont\small Explanations:  
        No arguments or attacks in an agent's local scope 
     will be removed.}  
    \vspace{0.14cm} 
        \item For $e_1, e_2 \in E_1$ and $a \in {\fa}_1(e_1)$, 
            if $(\top, a) \in {\fleqintra}_1(e_1, e_2)$, then 
          $(\top, a) \in {\fleqintra}_2(e_1, e_2)$ 
      (\textbf{monotonic facts}). {\normalfont\small Explanations: No argument 
         that an agent considers factual to some agent will become 
         non-factual.}  
     \vspace{0.14cm} 
        \item ${\ve}_1 = {\ve}_2$ (\textbf{constant $\ve$}).  
   {\normalfont\small Explanations:  $\revV$ handles 
  agent-to-agent
        trusts.} 
     \end{enumerate}
     We say that $F^{p\Dung} \in \mathcal{F}^{p\Dung}$ 
     is a public announcement to $F^{MA}_1 \in \mathcal{F}^{MA}$ 
      iff 
     $\announce(F^{MA}_1, F^{p\Dung})$ is defined. 
   \end{definition}        
   These conditions are sufficient to 
   ensure no change in $\fe(e)$, $e \in E_1$, 
   when a public announcement does not   
   contain an argument in $\fe(e)$: 
   \begin{theorem}[Local scope preservation] 
       For $F_1^{MA} \equiv 
        ({\Fg}_1, {\Fpub}_1, E_1, {\fe}_1, {\fa}_1, {\gsem}_1, 
   {\fleqintra}_1,\linebreak {\ve}_1, {\fleqinter}_1)$ and 
        $F^{p\Dung} \in \mathcal{F}^{p\Dung}$, 
        if $F^{p\Dung}$ is a public announcement to 
        $F^{MA}_1$ with
      $\announce(F^{MA}_1, F^{p\Dung})\linebreak
        = ({\Fg}_2, {\Fpub}_2, E_2, {\fe}_2, {\fa}_2, {\gsem}_2,
       {\fleqintra}_2, {\ve}_2, {\fleqinter}_2)$, then  
        for every $e \in E_1$,  
        if $\getArg(F^{p\Dung}) \cap \getArg({\fe}_2(e)) = \emptyset$, 
        then ${\fe}_1(e) = {\fe}_2(e)$. 
   \end{theorem}  
\begin{example}[Continued from Example 1]  
   For the transition from \fbox{D} to
   \fbox{E}, suppose 
  $F^{p\Dung} \equiv (\{a_5\}, \{(a_5,a_2), (a_5, a_3)\})$.   
  Since $a_2, a_3 \in \getArg(\Fpub)$ $\andC$ 
  $(a_5, a_2), (a_5, a_3) \not\in  \getR(\Fpub)$, 
  both (\textbf{no leak}) and (\textbf{no repetition}) 
  are satisfied, and $F^{p\Dung}$ is a public announcement 
  to $F^{MA}$, such that $\gsem = {\gsem}_x$ (by (\textbf{same 
   semantics})) and that $\ve = {\ve}_x$ (by (\textbf{constant 
 $\ve$})). 
\end{example} 

 \subsection{Updates: deception/honesty detection and 
   revision of agent-to-agent 
    trusts.}  
  As per the discussion in Section~\ref{section_motivation}, 
  each agent $e_1 \in E$ carries out deception/honesty detection of 
  some set of arguments in another agent $e_2$'s scope, 
  which it does in the steps we described (see section~\ref{subsection_intra_agent_preferences}), based on trust-neutral semantics 
   (Definition~\ref{def_trust_neutral_semantics}).  \\ 
    \indent For notational 
   convenience in the definition of deception/honesty detection below,  
   we define  
   a function $\sub: 2^{2^{\mathcal{A}}} \times 2^{\mathcal{A}}
   \rightarrow 2^{2^{\mathcal{A}}}$ which is such that   
    $\sub(\Gamma, A_z) = \{A_x \in 2^{\mathcal{A}} \ | \ 
        A_y \in \Gamma \ \andC \  A_x = (A_y \cap A_z)\}$. 
    This is a filtering function on 
    each member of $\Gamma$ to the members of $A_z$. 
    We write $\Gamma_{\downarrow A_z}$ as a shorthand of 
    $\sub(\Gamma, A_z)$.  
  \begin{definition}[Deception/honesty detection]\label{def_deception_honesty_detection} 
      Let $\detect: \mathcal{F}^{MA} \times \mathcal{E} \times 
   \mathcal{E} \times \mathcal{F}^{p\Dung} 
    \rightarrow \{hnst, dishnst, ?\}$ be such that 
     for any $F^{MA} \equiv (\Fg, \Fpub, E, \fe, \fa, 
     \gsem, \fleqintra, \ve, \fleqinter)$ and any $e_1, e_2 \in E$, 
     $\detect(F^{MA}, e_1, e_2, F^{p\Dung})$ is defined iff   
     $F^{p\Dung}$ is a public announcement to $F^{MA}$. 

     In case it is defined, let $F^{MA}_2 \equiv 
     ({\Fg}_2, {\Fpub}_2, E_2, {\fe}_2, {\fa}_2,  
     {\gsem}_2, {\fleqintra}_2, {\ve}_2, {\fleqinter}_2)$ be such that 
     $\announce(F^{MA}, F^{p\Dung}) = (F^{MA}, F^{p\Dung}, F^{MA}_2)$,  
     $\Gamma^{ntrl}_{\pub_{\fleqintra(e_1, e_2)}}$ denote
     $e_1$'s model of $e_2$'s public semantics, 
     $\Gamma^{ntrl}_{e_1 \smile_{\fleqintra} e_2}$ denote 
     $e_1$'s model of $e_2$'s local agent semantics, 
    and\linebreak $\detect(F^{MA}, e_1, e_2, F^{p\Dung})$ be:
     \begin{itemize}[leftmargin=0.3cm]
        \item $dishnst$ if  
    for every $A_x \in (\Gamma_{\pub_{\fleqintra(e_1,e_2)}}^{ntrl})_{\ 
  \downarrow 
         \getArg(F^{p\Dung} \cap {\fe}_2(e_2))}$ there exists 
     no  
     $A_y \in\linebreak (\Gamma_{e_1 \smile_{\fleqintra} e_2}^{ntrl})_{\ 
  \downarrow 
         \getArg(F^{p\Dung} \cap {\fe}_2(e_2))}$
   such that 
        $A_x = A_y$.  
   \begin{adjustwidth}{0.2cm}{}
     {\normalfont\small Explanations: as given 
   in \ref{subsection_intra_agent_preferences}. $e_1$ sees $e_2$  
    considers: (1) 
  each member of $\Gamma^{ntrl}_{\pub_{\fleqintra(e_1, e_2)}}$
    possibly publicly acceptable; and (2) each member of 
     $\Gamma^{ntrl}_{e_1 \smile_{\fleqintra} e_2}$  
    $e_2$ actually considers possibly acceptable.  
    Thus, deception by $e_2$ is not certain 
    if there is an overlap between the two semantics, 
    even if the two do not exactly match. Clear, on the other hand,
    when there is no overlap. 
    } 
   \end{adjustwidth} 
   \vspace{0.14cm} 
       \item $hnst$ if $(\Gamma_{\pub_{\fleqintra(e_1,e_2)}}^{ntrl})_{\ 
  \downarrow 
         \getArg(F^{p\Dung} \cap {\fe}_2(e_2))} =  
   (\Gamma_{e_1 \smile_{\fleqintra} e_2}^{ntrl})_{\ 
  \downarrow 
         \getArg(F^{p\Dung} \cap {\fe}_2(e_2))}$ $\andC$ \linebreak
    $(\top, a) \in \fleqintra(e_1, e_2)$ 
       for each $a \in \getArg(F^{p\Dung}) \cap \getArg({\fe}_2(e_2))$.   
     \begin{adjustwidth}{0.2cm}{} 
     {\normalfont\small Explanations:  
        as given in \ref{subsection_intra_agent_preferences}. 
        When the two semantics restricted to  
        the newly announced arguments 
        do not exactly match, there is a possibility 
        of information withholding, which may be potentially
        intentional. When what $e_2$ publicly claims acceptable
        (from $e_1$'s perspective) are not factual,  
        they are not proved truthful to arguments known to 
        be factual. See also the footnote 4. 
      }   
     \end{adjustwidth} 
    \vspace{0.14cm} 
     \item $?$, otherwise. 
     \end{itemize}
  \end{definition} 
  This definition is a faithful formalisation of 
  the detection steps in section~\ref{subsection_inter_agent_preferences}
  where examples are found. We obtain: 
   \begin{definition}[Revision of agent-to-agent 
     trusts]\label{def_revision} 
       Let $\revV: \mathcal{F}^{MA} \times 
    \mathcal{F}^{p\Dung} \times \mathcal{F}^{MA} 
      \rightarrow \mathcal{F}^{MA}$     
      be such that     
    $F^{MA}_3 \equiv 
    \revV(F^{MA}_1, F^{p\Dung}, F^{MA}_2)$ satisfies:   
   \vspace{-0.2cm} 
%
      \begin{itemize} 
         \item  
        ${\ve}_2(e_1, e_2) = {\ve}_3(e_1, e_2)$ 
          if $\detect(F^{MA}_1, e_1, e_2, F^{p\Dung}) =\ ?$ (suppose 
          $E_2$ ($E_3$) occurs in $F^{MA}_2$ ($F^{MA}_3$), 
           and that $e_1, e_2$ are the members of $E_2$; similarly
           for the two cases below).  
        \item ${\ve}_2(e_1, e_2) \leq {\ve}_3(e_1, e_2)$ 
          if $\detect(F^{MA}_1, e_1, e_2, F^{p\Dung}) = hnst$. 
        \item ${\ve}_3(e_1, e_2) \leq {\ve}_2(e_1, e_2)$ 
          if $\detect(F^{MA}_1, e_1, e_2, F^{p\Dung}) = dishnst$. 
      \end{itemize}   
   \end{definition}  
   Thus, honesty (and deception), if detected, 
   does not have a negative (and a positive) impact on 
  perceived trustworthiness. It is left open by how much 
   the numerical values might change. Finally: 
   \begin{definition}[Updates]\label{def_updates}  
      Let $\update: \mathcal{F}^{MA} \times \mathcal{F}^{p\Dung} 
     \rightarrow \mathcal{F}^{MA}$ be such that 
    $\update = (\revV \circ \announce)$.  
     We say that $\update(F^{MA}, F^{p\Dung})$ is 
    an updated manipulable multi-agent argumentation 
   of $F^{MA}$ with $F^{p\Dung}$ iff 
    $F^{p\Dung}$ is a public announcement to $F^{MA}$. 
   \end{definition}          
  A difference is clearly observable before and after 
  an update: 
\begin{proposition}[Existence of an update] 
     For every $F^{MA} \in \mathcal{F}^{MA}$ and every
     $F^{p\Dung} \in \mathcal{F}^{p\Dung}$,  
    $\update(F^{MA}, F^{p\Dung}) \not= F^{MA}$ 
     if $F^{p\Dung}$ is a public announcement to  $F^{MA}$. 
\end{proposition}    
  \hide{  
   \begin{definition}[Updating function]  
       We define $\Rightarrow: \mathcal{F}^{MA} \rightarrow \mathcal{F}^{MA}$
       to be an update function. We say that 
    $\Rightarrow(F^{MA})$ is an updated multi-agent 
    argumentation of $F^{MA}$. 
   \end{definition} 
   \begin{definition}[Postulates on $\Rightarrow$] 
        We presume the following postulates on $\Rightarrow$: 
     \begin{itemize} 
          \item   
     \end{itemize} 
   \end{definition}   
}    
\section{Related Work and Conclusion} 
\textbf{Related work.} Deceptive or manipulable 
argumentation in multi-agent argumentation has been sporadically 
studied. A notion of deception detection within argumentation 
was defined in 
\cite{Sakama12}. We have taken detailed comparisons to 
it in Section \ref{section_motivation}, and proposed an 
alternative approach with 
intra-agent preferences and epistemic states.   
A variation of \cite{Sakama12} is found in 
 \cite{Takahashi16} that relaxes    
the assumption of agents' attack-omniscience (Cf.   
Section~\ref{section_motivation}). Several 
other 
assumptions are made, however, including: (1) restriction of $e_1 \smile e_2$ argumentation 
that it be a sub-argumentation of $\fa(e_2)$, under which 
$e_1$'s belief of $e_2$'s knowledge of $e_3$'s argument(ations) 
is not expressible when $e_2 \not= e_3$ and 
when $\getArg(\fa(e_2)) \cap \getArg(\fe(e_3)) = \emptyset$;   
(2) grounded semantics for every agent;
(3) two-party and turn-based dialogue (also in \cite{Sakama12}); 
and (4) an update on agent's local argumentation by argument(s), 
not by a general argumentation (also in \cite{Sakama12}).  
Our $\mathcal{F}^{MA}$ generalises 
on all these points.\footnote{While  
the examples that we used in this paper are 
turn-based, it is clear that $\mathcal{F}^{MA}$  
permits a public announcement simultaneously by 
more than one agent.}  
$\mathcal{F}^{MA}$ also models consequence of 
deception/honesty detection on agent-to-agent trusts, which 
is not discussed in \cite{Sakama12,Takahashi16}, as far as we are aware. 
In \cite{Kuipers10}, exploitation of agent's logical inference 
capacity was investigated for logic-based argumentation 
as opposed to abstract argumentation we studied in this paper. 
In \cite{Kontarinis15}, each agent which 
maintains its own argumentation for a certain discussion issue is 
defined: lying if it announces an argument-to-argument 
relation to an argument when the relation is not in its 
argumentation; and hiding 
if it does not announce an argument-to-argument relation to 
an argument in 
its argumentation even though the announcement would alter 
evaluation of the argument. The objective of \cite{Kontarinis15} is 
to measure how accurately the lie/hiding can be 
estimated from agent's active participation 
to discussion, changes in evaluation the agent makes to arguments, 
and argument-graph similarity among agents (for the last, 
see also \cite{Hadjinikolis13}). However,
agents' argumentations are assumed constant; 
epistemic states are not discussed; detection of the lie/hiding 
is based on the three hypothesised criteria and not on 
agents' acceptability semantics; and the impact of 
detected lie/hiding is not covered, in \cite{Kontarinis15}.

Generally speaking, studies on 
persuasions and dialogue games \cite{Kakas05,Sakama12,Grossi13,Rahwan09,Parsons05,Riveret08,Hadjinikolis13,Hadoux15,Hadoux17,Hunter18} focus on 
turn-based 2-party dialogues with single-argument public announcements 
at each turn. 
Restriction of agents' semantics to grounded semantics 
is also dominant due to ease of treating it.  
For strictly more than 2-party dialogue games, 
there is for instance \cite{Grossi13}, with 
2-parties plus an observer, whereby the 2-parties playing 
a complete information game against each other 
are uncertain about how the observer evaluates 
their argumentation. In $\mathcal{F}^{MA}$, there is no 
particular distinction among agents, for every agent 
is a player and an observer. Further, information is generally 
incomplete for each agent. Non-turn-based (generally) non-2-party 
concurrent persuasion is studied in \cite{ArisakaSatoh18}, 
where defence against persuasion 
is considered. No consideration, however, was given 
of agents' epistemic states, and the concept of agency, 
while certainly expressible, 
is still not explicit in \cite{ArisakaSatoh18}.  
For agent-to-agent epistemic relations 
for incomplete multi-agent argumentation, 
there is some work \cite{ArisakaBistarelli18} on 
outsourcing of defence to other agents. The relations, however, 
are static, while our relations are responsive to announced 
argumentations. The need for more than one reasoning 
  mode in multi-agent argumentation
  was  
  recognised a while back \cite{Kakas05} for adapting to normal 
  and exceptional circumstances. In this work,   
  we demonstrated with intra-agent preferences that 
  a similar need arises in manipulable argumentation, which 
  has not been observed in previous studies. 
   Moreover, we introduced 
  inter-agent preferences to express trusts 
  and their impacts on agents' acceptability judgement, 
   which is not in \cite{Kakas05} as far as our understanding goes.

\indent \textbf{Conclusion.} We presented a manipulable 
   multi-agent argumentation theory with epistemic agents, 
  and formulated deception/honesty detection 
  as well as their 
  influence on trusts. 
Experimental results 
  will be 
  found in an extended work. For a formal interest,
  relaxation of Definition
  \ref{def_manipulable_multiagent_argumentation} (\textbf{binary}
  in particular) 
  and Definition \ref{def_public_announcements} (\textbf{same agents}, 
  \textbf{same semantics}, in particular) can be studied. 
   \bibliography{references} 
  \bibliographystyle{abbrv}   
\section*{Appendix: Proofs}  
    \setcounter{proposition}{0} 
   \begin{proposition}  
 $(\getR(\fa(e)) \cup \getR(\Fg)) 
\cap (\getArg(\fe(e)) \times     
   \getArg(\fe(e))) =  \getR(\fe(e))$.  
   \end{proposition}  
   \begin{proof}  
      We have $\fa(e) \in 2^{\Fg}$, and therefore 
      $\getR(\fa(e)) \cup \getR(\Fg) = 
       \getR(\Fg)$. Together with (\textbf{local scopes}),  
      we obtain $(\getR(\fa(e)) \cup \getR(\Fg))
   \cap (\getArg(\fe(e)) \times \getArg(\fe(e))) = 
      \getR(\fe(e))$, as required. 
    \hfill$\Box$ 
   \end{proof}  
 \setcounter{theorem}{0} 
  \begin{theorem}[Existence]  
       $\mathcal{F}^{MA} \not= \emptyset$.   
  \end{theorem} 
  \begin{proof} 
       We have examples of $\mathcal{F}^{MA}$ in this paper.
       See Example 1. 
      \hfill$\Box$ 
  \end{proof} 
  \begin{theorem}[Local scope preservation] 
    For $F_1^{MA} \equiv 
        ({\Fg}_1, {\Fpub}_1, E_1, {\fe}_1, {\fa}_1, {\gsem}_1, 
   {\fleqintra}_1,\linebreak {\ve}_1, {\fleqinter}_1)$ and 
        $F^{p\Dung} \in \mathcal{F}^{p\Dung}$, 
        if $F^{p\Dung}$ is a public announcement to 
        $F^{MA}_1$ such that 
      $\announce(F^{MA}_1, F^{p\Dung})\linebreak 
        = ({\Fg}_2, {\Fpub}_2, E_2, {\fe}_2, {\fa}_2, {\gsem}_2,
       {\fleqintra}_2, {\ve}_2, {\fleqinter}_2)$, then  
        for every $e \in E_1$,  
        if $\getArg(F^{p\Dung}) \cap \getArg({\fe}_2) = \emptyset$, 
        then ${\fe}_1(e) = {\fe}_2(e)$. 
  \end{theorem}   
  \begin{proof} 
      A consequence of (\textbf{local scope monotonicity}), 
     (\textbf{expansion}), and the fact that 
      $F^{p\Dung}$ is a member of $\mathcal{F}^{p\Dung}$, 
      i.e. $a_1, a_2 \in \getArg(F^{p\Dung} \cup {\Fpub}_1)$ for every 
       $(a_1, a_2) \in  \getR(F^{p\Dung})$.  \hfill$\Box$ 
   \end{proof}  

\setcounter{proposition}{1} 
\begin{proposition}[Existence of an update] 
    For every $F^{MA} \in \mathcal{F}^{MA}$ and every
     $F^{p\Dung} \in \mathcal{F}^{p\Dung}$,  
    $\update(F^{MA}, F^{p\Dung}) \not= F^{MA}$ 
     if $\announce(F^{MA}, F^{p\Dung})$ is defined. 

\end{proposition}     
\begin{proof} 
    If $\announce(F^{MA}, F^{p\Dung})$ is defined, 
    $(F^{MA}, F^{p\Dung}, F_2^{MA}) \equiv \announce(F^{MA},
   F^{p\Dung})$ is such that
    ${\Fpub}_1 \not= {\Fpub}_2$ for ${\Fpub}_i$ 
     that occurs
    in $F^{MA}_i$, $i \in \{1,2\}$.
    Since $\announce(F^{MA}, F^{p\Dung})$ is defined,
    so is also $\update(F^{MA}, F^{p\Dung})$ by definition. 
    It is also by definition that 
    ${\Fpub}_3$ occurring in $\update(F^{MA}, F^{p\Dung})$
    is such that ${\Fpub}_2 = {\Fpub}_3$. \hfill$\Box$ 
\end{proof} 
%
%
\end{document}